\documentclass{article}

\usepackage[preprint]{neurips_2026}
\usepackage[normalem]{ulem}

\usepackage{amsmath,amssymb,amsfonts,amsthm,mathtools}
\usepackage{multirow}
\usepackage{aliascnt}

\usepackage{algorithm,algpseudocode}
\algblock{ParFor}{EndParFor}
\algnewcommand\algorithmicparfor{\textbf{parfor}}
\algnewcommand\algorithmicpardo{\textbf{do}}
\algnewcommand\algorithmicendparfor{\textbf{end parfor}}
\algrenewtext{ParFor}[1]{\algorithmicparfor\ #1\ \algorithmicpardo}
\algrenewtext{EndParFor}{\algorithmicendparfor}

\usepackage{booktabs,graphicx,subcaption}
\usepackage{array}
\usepackage{xcolor}
\usepackage{fontawesome5}
\newcommand{\hl}[1]{#1}
\newcommand{\hlpar}[1]{#1}

\usepackage{hyperref}
\hypersetup{colorlinks=true,linkcolor=blue,urlcolor=blue,citecolor=blue}
\usepackage[capitalize,noabbrev]{cleveref}
\crefname{assumption}{Assumption}{Assumptions}
\crefname{theorem}{Theorem}{Theorems}
\crefname{lemma}{Lemma}{Lemmas}
\crefname{proposition}{Proposition}{Propositions}

\usepackage{enumitem}
\usepackage{nicefrac}
\usepackage{microtype}

\newtheorem{theorem}{Theorem}

\newaliascnt{lemma}{theorem}
\newtheorem{lemma}[lemma]{Lemma}
\aliascntresetthe{lemma}

\newaliascnt{proposition}{theorem}
\newtheorem{proposition}[proposition]{Proposition}
\aliascntresetthe{proposition}

\newaliascnt{corollary}{theorem}
\newtheorem{corollary}[corollary]{Corollary}
\aliascntresetthe{corollary}

\theoremstyle{definition}
\newaliascnt{definition}{theorem}

\aliascntresetthe{definition}

\newtheorem{assumption}{Assumption}

\theoremstyle{remark}
\newtheorem{remark}{Remark}

\newcommand{\R}{\mathbb{R}}

\newcommand{\N}{\mathbb{N}}
\newcommand{\bfX}{\mathbf{X}}

\newcommand{\bfY}{\mathbf{Y}}
\newcommand{\bfW}{\mathbf{W}}
\newcommand{\bfA}{\mathbf{A}}
\newcommand{\bfB}{\mathbf{B}}
\newcommand{\bfM}{\mathbf{M}}
\newcommand{\bfN}{\mathbf{N}}
\newcommand{\bfI}{\mathbf{I}}
\newcommand{\bfa}{\mathbf{a}}
\newcommand{\bfb}{\mathbf{b}}
\newcommand{\bfu}{\mathbf{u}}
\newcommand{\bfv}{\mathbf{v}}

\newcommand{\bfDelta}{\boldsymbol{\Delta}}
\newcommand{\bftheta}{\boldsymbol{\theta}}
\newcommand{\calE}{\mathcal{E}}
\newcommand{\calF}{\mathcal{F}}
\newcommand{\calT}{\mathcal{T}}
\newcommand{\calL}{\mathcal{L}}
\newcommand{\calO}{\mathcal{O}}
\newcommand{\norm}[1]{\left\|#1\right\|}

\newcommand{\svmax}{\sigma_{\max}}
\newcommand{\svmin}{\sigma_{\min}}
\newcommand{\topone}{\texttt{Rank1}}
\newcommand{\What}{\hat{W}}
\newcommand{\calB}{\mathcal{B}}
\DeclareMathOperator*{\argmin}{arg\,min}
\newcommand{\adapad}{\textsc{AdaPaD}}

\title{\adapad{}: Adaptive Parallel Deflation for\\PEFT with Self-Correcting Rank Discovery}

\author{
  Barbara Su \quad Fangshuo Liao \quad Anastasios Kyrillidis \\
  Department of Computer Science \\
  Rice University \\
  \texttt{\{bs82,\,Fangshuo.Liao,\,anastasios\}@rice.edu}
}

\begin{document}
\maketitle
% First pass done.
\begin{abstract}
Fine-tuning large language models with LoRA~\citep{HuLoRA2022} requires choosing a rank $r$ before training starts.
Existing approaches either extract rank-1 components sequentially~\citep{LiaoErrorProp2023,VandchaliOneRank2024}, freezing each component's error permanently into every subsequent residual, or optimize the full low-rank factorization jointly~\citep{dEBORA2025,ARDLoRA2025} with guarantees that describe only the joint update, not individual rank-1 directions.
We present \adapad{} (Adaptive Parallel Deflation), which trains all rank-1 components simultaneously: each worker refines its component against a deflation target built from the latest estimates of all predecessors, and as those estimates improve, the targets improve too.
We call this property \emph{self-correction}: deflation errors converge to zero over rounds rather than persisting as fixed residuals.
On top of this backbone, \adapad{} adds \emph{advance learning} (private pre-training before activation) and \emph{per-module dynamic rank discovery} (importance-based growth until a shared budget is exhausted), making the rank distribution an output rather than an input.
We prove that every component's error decays exponentially after a warm-up period, with a generalization bound that splits into a vanishing algorithmic term and an irreducible statistical floor.
Empirically, \adapad{} is competitive with adaptive-rank LoRA baselines on GLUE with DeBERTaV3-base at matched parameter budgets, and competitive with fixed-rank LoRA on Qwen3-0.6B SQuAD/SQuAD~v2 while deploying an adapter that is on average $30.7\%$ smaller.~\href{https://anonymous.4open.science/r/ParallelLinearRegression-C604/README.md}{\faGithub}
\end{abstract}

% Done, except for contribution. Modify it after finalizing experimental section

\section{Introduction}\label{sec:intro}

Parameter-efficient fine-tuning (PEFT) is the standard way to adapt large pretrained models to downstream tasks. Among the many approaches~\citep{HoulsbyAdapter2019,LiLiangPrefix2021,ZakenBitFit2022}, Low-Rank Adaptation (LoRA)~\citep{HuLoRA2022} stands out for its simplicity: for a pretrained weight $\bfW \in \R^{m\times d}$, LoRA freezes $\bfW$ and trains a low-rank update $\Delta\bfW = \bfB\bfA^\top$ with $\bfB \in \R^{m\times r}$, $\bfA \in \R^{d\times r}$ and rank $r \ll \min(m,d)$ a hyperparameter.
LoRA is effective in practice, but leaves a fundamental design choice unresolved: how should $r$ be chosen?

In practice $r$ is set by hand and applied uniformly to every layer~\citep{HuLoRA2022,HayouLoRAPlus2024,LiuDoRA2024}, yet the appropriate rank varies across modules, tasks, and training stages.
Adaptive-rank methods address this in different ways: AdaLoRA~\citep{ZhangAdaLoRA2023} prunes from high rank using SVD importance, IncreLoRA~\citep{ZhangIncreLoRA2023} grows from low rank by activating reserve components, DyLoRA~\citep{ValipourDyLoRA2023} trains across a range of ranks, and SoRA~\citep{DingSoRA2023} learns sparse gates.
Recent theoretical work analyzes the joint $\bfB\bfA^\top$ update: dEBORA~\citep{dEBORA2025} uses bilevel Frank-Wolfe with rank-identifiability guarantees, and ARD-LoRA~\citep{ARDLoRA2025} adds learnable per-head scaling with a stationary-point theorem.
Both treat the full low-rank matrix as one object; we instead decompose the problem into independent rank-1 subproblems trained in parallel and analyze each separately, yielding a guarantee joint optimization cannot express: \emph{self-correction}, where the deflation mismatch in each rank-1 worker provably vanishes over rounds, so individual rank-1 directions can be certified correct rather than only their joint sum.

\textbf{Key insight: parallel deflation with self-correction.}
In sequential deflation, errors in early components are permanently frozen and compound multiplicatively through the spectral gap~\citep{VandchaliOneRank2024}.
We remove this dependency by training all rank-1 components \emph{simultaneously}: each is assigned to a dedicated worker that, at every communication round, rebuilds its deflation target from the \emph{latest} predecessor estimates, so early mistakes are progressively corrected rather than locked in.
We call this property \emph{self-correction} and prove that the deflation mismatch $\|\bfY_{k,\ell} - \bfY_k^\star\|_F$ vanishes as $\ell$ grows (\cref{prop:self_correction}), in contrast to sequential deflation~\citep{VandchaliOneRank2024,LiaoErrorProp2023} where the mismatch is fixed and generically nonzero.

The bilinear setting differs from symmetric eigendecomposition~\citep{LiaoParallelPCA2025}: the asymmetric structure $\bfb\bfa^\top\bfX$ calls for Wedin perturbation theory~\citep{Wedin1972} and a data-dependent recovery scale. Two further mechanisms make parallel deflation practical for LoRA: \emph{advance learning} (private warm-up of dormant components against improving targets) and \emph{per-module dynamic rank discovery} (importance-driven growth under a global budget).

\textbf{Contributions.}
\begin{enumerate}[leftmargin=*,label=\textbf{C\arabic*.}, nosep]
\item \textbf{\adapad{} algorithm} (\cref{alg:parallel_rank1,alg:adapad}). Parallel rank-1 deflation with \emph{advance learning} (private pre-training before activation) and \emph{per-module dynamic rank discovery} (importance-based growth under a shared budget); the rank distribution is an output, not an input.
\item \textbf{Theoretical guarantees.}
Exponential convergence with vanishing deflation mismatch (\cref{thm:parallel_convergence,prop:self_correction}) and a weight-recovery bound that splits into a vanishing algorithmic term plus an irreducible statistical floor (\cref{thm:gen_noiseless_parallel,thm:gen_noisy_parallel}); to our knowledge, the first self-correcting guarantee for parallel rank-1 deflation in bilinear regression, structurally distinct from the joint-update guarantees of~\citep{dEBORA2025,ARDLoRA2025}.
\item \textbf{Empirical evaluation.}
At approximately matched budgets ($0.32$--$0.47$\,M) on GLUE with DeBERTaV3-base, \adapad{} is competitive with IncreLoRA, AdaLoRA, SoRA, and dEBORA across all eight tasks (\cref{sec:exp_glue}); on Qwen3-0.6B SQuAD/SQuAD~v2 it matches fixed-rank LoRA on F1/EM with adapters $30.7\%$ smaller on average (\cref{sec:exp_squad}); component-parallel sharding delivers up to $2.66{\times}$ speedup on $4{\times}$H200 (\cref{sec:exp_qwen_scaling}).
\end{enumerate}

\textbf{Related work.}

\emph{Deflation-based methods.}
Hotelling's deflation~\citep{Hotelling1933} is the classical approach for extracting principal components; \citet{Mackey2008} extended it to sparse PCA.
Error propagation in inexact sequential deflation was analyzed by \citet{LiaoErrorProp2023} and \citet{VandchaliOneRank2024}.
The model-parallel perspective was introduced by EigenGame~\citep{GempEigenGame2021,GempEigenGameUnloaded2022}, which casts top-$K$ PCA as a multi-player game.
\citet{LiaoParallelPCA2025} gave the first provable parallel deflation algorithm for PCA; we develop parallel deflation for the bilinear regression setting, where the asymmetric structure $\bfb\bfa^\top\bfX$ requires Wedin-type perturbation analysis and the self-correction guarantee is tied to a data-dependent recovery scale.
ReLoRA~\citep{LialinReLoRA2023} trains a sequence of low-rank updates that sum to a high-rank result, which is sequential deflation under a different name; our approach differs by running all components in parallel with provable self-correction.

\emph{Low-rank optimization.}
Problem~\eqref{eq:main_problem} has been attacked via nuclear norm minimization~\citep{CandesRecht2009,Recht2011}, projected gradient descent~\citep{JainGD2010,KyrillidisIHT2014}, and factored methods~\citep{BurerNonlinear2003,TuLowRank2016,KyrillidisProvable2018, park2017non, bhojanapalli2016dropping, park2018finding, kim2023fast, geyer2020low, hsieh2018non}.
These optimize the full low-rank matrix jointly but do not decompose into independent rank-1 subproblems.
\citet{LoRAConvergence2025} show that fixed-rank LoRA training converges to a low-rank global minimum or diverges detectably, but this addresses the fixed-rank setting without rank adaptation.

\emph{Adaptive-rank methods with convergence theory.}
dEBORA~\citep{dEBORA2025} formulates rank selection as bilevel optimization: an upper-level Frank-Wolfe on an $\ell_1$-constrained simplex selects singular-value scaling factors while a lower-level solver trains the basis matrices.
dEBORA proves $\calO(1/\sqrt{T})$ convergence of the upper level and, notably, finite-time rank identification (the algorithm provably discovers which singular values are zero).
As a joint optimization method, dEBORA does not decompose into rank-1 subproblems and therefore does not face the deflation error propagation problem; its convergence guarantee applies to the upper-level simplex variable, not to the bilinear $\bfB\bfA^\top$ factorization itself.
ARD-LoRA~\citep{ARDLoRA2025} introduces learnable per-head scaling factors with $\calO(1/T)$ convergence to a stationary point via a standard descent lemma; the guarantee is generic smooth optimization and does not exploit the bilinear structure. Neither decomposes into rank-1 subproblems nor exhibits self-correction (\cref{prop:self_correction}).

\section{Background and Problem Setup}\label{sec:prelim}

\textbf{Low-Rank Regression and LoRA.}\label{sec:setup}
We study the linear model $\bfY = \bfW^\star \bfX + \calE$,
where $\bfW^\star \in \R^{m \times d}$ has rank $r^\star$, $\bfX \in \R^{d \times n}$ collects input features, $\bfY \in \R^{m \times n}$ is the response, and $\calE$ is noise ($\calE = \mathbf{0}$ in the noiseless case; $\calE_{ij} \sim \mathcal{N}(0, \varepsilon^2)$ in the noisy case).
The goal is to solve
\begin{equation}\label{eq:main_problem}
\min_{\bfA \in \R^{r \times d},  \bfB \in \R^{m \times r}} \frac{1}{2}\norm{\bfY - \bfB\bfA\bfX}_F^2.
\end{equation}
This directly models LoRA: for a pretrained layer with weight $\bfW_0$, taking $\bfX$ as the input activations and $\bfY = \Delta\bfW \cdot \bfX$ reduces LoRA training to~\eqref{eq:main_problem} with $\bfW^\star = \Delta\bfW = \bfB\bfA^\top$.

\textbf{Sequential vs.\ Parallel Deflation.}
\emph{Sequential deflation} extracts components one at a time: set $\bfY_1=\bfY$, compute $(\bfa_k, \bfb_k) = \topone(\bfY_k, \bfX, T_k)$ with budget $T_k$, then $\bfY_{k+1} = \bfY_k - \bfb_k \bfa_k^\top \bfX$. Two defects: wall-clock time $\sum_k T_k$, and frozen errors that compound multiplicatively~\citep{VandchaliOneRank2024}.
\emph{Parallel deflation} assigns each component to its own worker and indexes rounds by $\ell = 1, \ldots, L$. At round $\ell$, worker $k$ holds $(\bfa_{k,\ell}, \bfb_{k,\ell})$ and uses the latest predecessor estimates to recompute $\bfY_{k,\ell} := \bfY - \sum_{k'<k}\bfb_{k',\ell-1}\bfa_{k',\ell-1}^\top\bfX$. The starred quantities $(\bfa_k^\star, \bfb_k^\star)$ denote the \emph{ideal} components of the clean target (\cref{sec:error_decomp}). \cref{tab:deflation_targets} contrasts the three regimes: the parallel target $\bfY_{k,\ell}$ improves over rounds rather than being fixed.

\begin{table}[h]
\centering
\small
\renewcommand{\arraystretch}{1.2}
\caption{Deflation targets for component $k$ under three deflation regimes. Exact sequential~\citep{Hotelling1933,Mackey2008} uses ideal predecessor outputs (unavailable in practice, since they require ground-truth singular vectors); inexact sequential~\citep{LiaoErrorProp2023,VandchaliOneRank2024} freezes approximate predecessors; parallel deflation refreshes the target every round with the latest predecessor estimates, so the target itself converges as $\ell\to\infty$.}\label{tab:deflation_targets}
\begin{tabular}{lll}
\toprule
\textbf{Method} & \textbf{Deflation target for component $k$} & \textbf{Property} \\
\midrule
Exact Sequential & $\bfY_k^\star = \bfY - \sum_{k'<k} \bfb_{k'}^\star \bfa_{k'}^{\star\top} \bfX$ & Fixed, ideal \\
Inexact Sequential & $\bfY_k = \bfY - \sum_{k'<k} \bfb_{k'} \bfa_{k'}^\top \bfX$ & Fixed, approximate \\
\textbf{Parallel (Ours)} & $\bfY_{k,\ell} = \bfY - \sum_{k'<k} \bfb_{k',\ell-1} \bfa_{k',\ell-1}^\top \bfX$ & \textbf{Improves over rounds} \\
\bottomrule
\end{tabular}
\end{table}

\textbf{Ideal Rank-1 Fit and Error Decomposition.}\label{sec:error_decomp}
Fix $k \in [r]$ and round $\ell$. The \emph{ideal rank-1 fit} of $\bfY_{k,\ell}$ is its best rank-1 Frobenius approximation,
\begin{equation}\label{eq:ideal_fit}
(\bar{\bfa}_{k,\ell}, \bar{\bfb}_{k,\ell}) \in \argmin_{\bfa \in \R^d,  \bfb \in \R^m} \norm{\bfY_{k,\ell} - \bfb\bfa^\top\bfX}_F,
\end{equation}
satisfying $\bar{\bfb}_{k,\ell}\bar{\bfa}_{k,\ell}^\top\bfX = \sigma_{1_k,\ell}\bfu_{1_k,\ell}\bfv_{1_k,\ell}^\top$ when the top triplet is unique. Define the \emph{clean} deflation target $\bfY_k^\star = \bfY - \sum_{k'<k}\bfb_{k'}^\star\bfa_{k'}^{\star\top}\bfX$ (exact sequential, top row of \cref{tab:deflation_targets}); its best rank-1 fit defines the \emph{ideal $k$-th component} $(\bfa_k^\star, \bfb_k^\star)$ with $\bfb_k^\star\bfa_k^{\star\top}\bfX = \sigma_k^\star\bfu_k^\star\bfv_k^{\star\top}$ via Eckart--Young--Mirsky~\citep{EckartYoung1936,Mirsky1960}.
Following the numerical-vs.-mismatch split of~\citep{LiaoErrorProp2023,VandchaliOneRank2024,LiaoParallelPCA2025}, we decompose worker $k$'s error at round $\ell$ into the \emph{numerical error}, \emph{deflation mismatch}, and \emph{total error}:
\begin{align}
D_{k,\ell} & \triangleq  \norm{\bfb_{k,\ell}\bfa_{k,\ell}^\top\bfX - \bar{\bfb}_{k,\ell}\bar{\bfa}_{k,\ell}^\top\bfX}_F, \quad \quad \text{(Numerical error)} \label{eq:def_D}\\
B_{k,\ell} & \triangleq  \norm{\bar{\bfb}_{k,\ell}\bar{\bfa}_{k,\ell}^\top\bfX - \bfb_k^\star\bfa_k^{\star\top}\bfX}_F,  \quad ~~\quad \text{(Deflation mismatch)} \label{eq:def_B}\\
G_{k,\ell} & \triangleq  \norm{\bfb_{k,\ell}\bfa_{k,\ell}^\top\bfX - \bfb_k^\star\bfa_k^{\star\top}\bfX}_F. \quad \quad ~~~\text{(Total error)}  \label{eq:def_G}
\end{align}
The triangle inequality gives $G_{k,\ell} \leq D_{k,\ell} + B_{k,\ell}$: $D_{k,\ell}$ vanishes as the subroutine iterates, $B_{k,\ell}$ shrinks as predecessors improve, and our proof shows both vanish as $\ell\to\infty$.

\subsection{Assumptions}\label{sec:assumptions}

\begin{assumption}[Rank-1 Subroutine Contraction]\label{asump:contraction}
For each component $k$, there exists a \emph{contraction factor} $\calF_k \in (0,1)$ such that the rank-1 subroutine, warm-started from $(\bfa_{k,\ell-1}, \bfb_{k,\ell-1})$, satisfies for all $\ell \geq k$:
$\norm{\bfb_{k,\ell}\bfa_{k,\ell}^\top\bfX - \bar{\bfb}_{k,\ell}\bar{\bfa}_{k,\ell}^\top\bfX}_F
\leq \calF_k \norm{\bfb_{k,\ell-1}\bfa_{k,\ell-1}^\top\bfX - \bar{\bfb}_{k,\ell}\bar{\bfa}_{k,\ell}^\top\bfX}_F$,
where $(\bar{\bfa}_{k,\ell}, \bar{\bfb}_{k,\ell})$ is the ideal rank-1 fit defined in~\eqref{eq:ideal_fit}.
\end{assumption}

\begin{assumption}[Spectral Gap]\label{asump:gap}
$\bfY = \bfW^\star\bfX$ has distinct singular values $\sigma_1^\star > \sigma_2^\star > \cdots > \sigma_r^\star > 0$ with gaps $\calT_k^\star := \min\{\min_{j>k}|\sigma_k^\star - \sigma_j^\star|, \sigma_k^\star\} > 0$.
We normalize $\sigma_1^\star = 1$.
\end{assumption}

\begin{assumption}[Bounded Rank-1 Recovery via Projection]\label{asump:bounded_recovery}
Parameter $Q$ in \Cref{alg:parallel_rank1} satisfies $Q \geq 2\sigma_1^\star$.
\end{assumption}

\hlpar{The projection in \cref{line:project} together with the assumption above implies the per-component magnitude bound $G_{k,\ell} \leq \mathcal{R}_k := Q + \sigma_k^\star$ used throughout the proofs (\cref{lem:bounded_recovery} in \cref{app:proofs}).}

All three assumptions are standard: \cref{asump:contraction} is the rank-1 subroutine regularity condition~\citep{LiaoParallelPCA2025}, holding for ALS ($\calF_k$ controlled by the local spectral ratio) and factored GD at small step size~\citep{JainNetrapalliSanghavi2013,ZhengLafferty2015,KyrillidisProvable2018,park2017non,bhojanapalli2016dropping,park2018finding,kim2023fast,geyer2020low,hsieh2018non} (pseudocode in \cref{app:subroutines}); \cref{asump:gap} is the spectral-gap condition for deflation~\citep{LiaoErrorProp2023,VandchaliOneRank2024,LiaoParallelPCA2025} ($\sigma_1^\star=1$ is a rescaling); \cref{asump:bounded_recovery} is automatic given the projection step in \cref{alg:parallel_rank1}, and the resulting magnitude bound (\hl{\cref{lem:bounded_recovery}}) is in the same spirit as bounded-iterate constructions in nonconvex factor recovery~\citep{TuBoczarSimchowitzSoltanolkotabiRecht2016,bhojanapalli2016dropping,park2018finding,GeJinZheng2017}.
In practice, \cref{asump:contraction} only needs to hold for $\ell \geq s_k$ (the warm-up round in \cref{thm:parallel_convergence} via \cref{eq:sk_condition}), by which point predecessor errors preserve the top-triplet structure of $\bfY_{k,\ell}$. Background perturbation results (Weyl, Wedin) are in \cref{app:background}.

\section{\adapad{}: The Algorithm}\label{sec:algorithm}

\begin{algorithm}[t]
\caption{Parallel Rank-1 Deflation for Low-Rank Regression}\label{alg:parallel_rank1}
\begin{algorithmic}[1]
\Require $\bfX \in \R^{d \times n}$, $\bfY \in \R^{m \times n}$, rank $r$, local iterations $T$, rounds $L \geq r$, projection radius $Q \geq \sigma_1^\star$
\Ensure Components $\{(\bfa_{k,L}, \bfb_{k,L})\}_{k=1}^r$
\For{$k = 1, \ldots, r$}
  \State Randomly initialize $(\bfa_{k,0}, \bfb_{k,0})$
\EndFor
\For{$\ell = 1, \ldots, L$}
  \ParFor{$k = 1, \ldots, r$} \Comment{All $r$ workers in parallel}
    \If{$k \leq \ell$} \Comment{Staggered activation}
      \State Receive $\{(\bfa_{k',\ell-1}, \bfb_{k',\ell-1})\}_{k'<k}$
      \State $\bfY_{k,\ell} \gets \bfY - \sum_{k'<k} \bfb_{k',\ell-1} \bfa_{k',\ell-1}^\top \bfX$ \label{line:deflate} \Comment{Deflate with latest estimates}
      \State $(\bfa_{k,\ell}, \bfb_{k,\ell}) \gets \topone(\bfY_{k,\ell}, \bfX, T, \bfa_{k,\ell-1}, \bfb_{k,\ell-1})$ \label{line:solve} \Comment{Warm-started}
      \State $\bfb_{k,\ell}\bfa_{k,\ell}^\top\bfX \gets \Pi_Q\!\bigl(\bfb_{k,\ell}\bfa_{k,\ell}^\top\bfX\bigr)$ \label{line:project} \Comment{Project onto Frobenius ball of radius $Q$}
      \State Broadcast $(\bfa_{k,\ell}, \bfb_{k,\ell})$
    \Else
      \State $(\bfa_{k,\ell}, \bfb_{k,\ell}) \gets (\bfa_{k,\ell-1}, \bfb_{k,\ell-1})$
    \EndIf
  \EndParFor
\EndFor
\end{algorithmic}
\end{algorithm}

\textbf{Base Algorithm: Parallel Rank-1 Deflation.}
\Cref{alg:parallel_rank1} parallelizes inexact sequential deflation~\citep{LiaoErrorProp2023,VandchaliOneRank2024,LiaoParallelPCA2025} for the bilinear model; we assume one worker per component for exposition (when $P<r$, components are sharded round-robin via a single \texttt{all\_gather} per round). Three mechanisms drive convergence: \emph{staggered activation} (worker $k$ begins at $\ell=k$), \emph{warm-starting} from $(\bfa_{k,\ell-1}, \bfb_{k,\ell-1})$, and \emph{deflation with current estimates}. The projection on \cref{line:project} caps each rank-1 estimate's Frobenius norm at $Q\geq\sigma_1^\star$, making the bounded-recovery bound \hl{\cref{lem:bounded_recovery}} automatic without affecting fixed points (the ideal estimate has norm $\sigma_k^\star\leq Q$). \hl{The projection in \cref{line:project} prevents the recovered component $\bfb_{k,\ell}\bfa_{k,\ell}^\top\bfX$ from blowing up while the deflated target $\bfY_{k,\ell}$ has not yet tightened around $\bfY_k^\star$ enough for the rank-1 subroutine to begin contracting. In practice we find that taking $Q=\infty$ (i.e., skipping the projection) suffices.} Two extensions, \emph{advance learning} and \emph{dynamic rank discovery}, follow.

\textbf{Advance Learning.}
In \cref{alg:parallel_rank1}, worker $k$ is idle for rounds $1, \ldots, k{-}1$. Inspired by IncreLoRA's reserve components~\citep{ZhangIncreLoRA2023}, we use this idle time: before activation, worker $k$ privately runs the rank-1 subroutine against its current deflation target without broadcasting, accumulating a good initialization. Unlike IncreLoRA's reserves, the warming component trains against the \emph{correct} target $\bfY_{k,\ell}$, which itself improves over rounds via self-correction.

% \textit{Learned component scaling.}
% Following IncreLoRA's per-component scalar scaling~\citep{ZhangIncreLoRA2023}, each rank-1 component carries a learnable scalar $\lambda_k$ that modulates its
% contribution: the adapter output for module $i$ with current active rank $r_i$ is
% $\Delta\bfW_i \bfx = \frac{\alpha}{r_i} \sum_{k=1}^{r_i} \lambda_{i,k}  \bfb_{i,k} \bfa_{i,k}^\top \bfx$,
% where the $\alpha/r_i$ scaling adapts automatically as modules grow.
% Committed components initialize $\lambda = 0$ (trainable), while reserve components
% are held at $\lambda = 10^{-5}$ (frozen) until activation.
% This soft gating lets a newly activated component's contribution grow smoothly from a near-zero magnitude to a magnitude on the order of the existing committed components, rather than being added abruptly at full magnitude, which would otherwise perturb the reconstruction.

Theoretically, advance learning tightens the convergence constant: the worst-case leading factor $3\mathcal{R}_k$ in \cref{thm:parallel_convergence} is replaced by $\mathcal{R}_k(\calF_k^{\,s_k-1}+2)$ when pre-activation contraction yields $D_{k,s_k-1}\ll\mathcal{R}_k$.

\textbf{Dynamic Rank Discovery.}
Attention and FFN sub-layers need different capacity, and the right allocation varies with depth and task~\citep{ZhangAdaLoRA2023,ZhangIncreLoRA2023,DingSoRA2023}, motivating per-module rank assignment.
\adapad{} (\hl{\cref{alg:adapad}}) maintains a per-module active rank $r_i$ (initially $1$, capped at $r_{\max}$), a shared budget $B$, and one private \emph{reserve} per module via advance learning. Every $\Delta$ mini-batches (one \emph{global round}, after a one-round warm-up), the top-$h$ eligible modules ($r_i < r_{\max}$) promote their reserve to a committed component and start advance learning on a new reserve; growth stops when $\sum_i r_i \geq B$.

\textit{Importance scoring.}
We adopt the importance-uncertainty product from IncreLoRA~\citep{ZhangIncreLoRA2023}: a module deserves capacity if its parameters are both large in magnitude and still receiving large gradient updates. With $\bftheta_i$ the committed $(\bfa,\bfb)$ entries of module $i$, the raw importance at step $t$ is
\begin{equation}\label{eq:raw_importance}
S_i^{(t)} = \frac{1}{\sqrt{\|\bftheta_i\|_2}} \sum_{p \in \bftheta_i} |p| \cdot |\nabla_p \calL|,
\end{equation}
normalized for comparability across module sizes. Two EMAs ($\beta_1{=}\beta_2{=}0.85$) track signal and volatility, and the allocation score is their product $\bar{S}_i^{(t)} U_i^{(t)}$:
\begin{equation}\label{eq:ema_ipt}
\bar{S}_i^{(t)} = \beta_1 \bar{S}_i^{(t-1)} + (1{-}\beta_1)  S_i^{(t)},\qquad U_i^{(t)} = \beta_2 U_i^{(t-1)} + (1{-}\beta_2)  |S_i^{(t)} - \bar{S}_i^{(t)}|.
\end{equation}

\begin{algorithm}[t]
\caption{\adapad{}: Parallel Deflation with Per-Module Rank Discovery}\label{alg:adapad}
\begin{algorithmic}[1]
\Require Model with $M$ adapter modules, max rank $r_{\max}$, budget $B$,
         growth interval $\Delta$, top-$h$, projection radius $Q\geq\sigma_1^\star$ (inherited by per-component updates as in \cref{alg:parallel_rank1})
\Ensure Per-module components $\{(\bfa_{i,k}, \bfb_{i,k})\}_{k=1}^{r_i}$ for $i = 1, \ldots, M$
\State $r_i \gets 1$ for all $i$; $R \gets M$ \Comment{Total active rank}
\For{round $= 1, \ldots, \lceil T_{\text{total}} / \Delta \rceil$}
  \State Cache $\Delta$ mini-batches $\calB = \{B_1, \ldots, B_\Delta\}$
  \State Snapshot committed components for all modules \Comment{Self-correction}
  \For{each batch $B_j \in \calB$}
    \ParFor{component $k$ owned by this worker}
      \State Forward with snapshot prefix $[0, k)$ + current component $k$
      \State Backward; update $(\bfa_{i,k}, \bfb_{i,k})$
    \EndParFor
  \EndFor
  \State Update importance EMAs $\bar{S}_i, U_i$ via~\eqref{eq:ema_ipt}
  \If{$R < B$}
    \State Select top-$h$ modules by $\bar{S}_i \cdot U_i$ (among those with $r_i < r_{\max}$)
    \For{each selected module $i$}
      \State Activate reserve $\to$ committed; $r_i \gets r_i + 1$; $R \gets R + 1$
      \State Spawn new reserve (advance learning)
    \EndFor
  \EndIf
\EndFor
\end{algorithmic}
\end{algorithm}

This per-module allocation is analogous to IncreLoRA's incremental growth but combined with parallel deflation and self-correction: committed components are snapshotted every $\Delta$ mini-batches, and new components train against improving deflation targets.

\section{Theoretical Guarantees}\label{sec:theory}

\textbf{Scope.}
Our analysis is formulated for the bilinear regression model $\bfY = \bfW^\star \bfX + \calE$ (\cref{eq:main_problem}), the exact structure of a single LoRA adapter layer; assumptions hold for ALS and factored GD~\citep{JainNetrapalliSanghavi2013,ZhengLafferty2015}. We treat each layer's subproblem in isolation and do not model the end-to-end nonlinear loss, as in dEBORA~\citep{dEBORA2025} and ARD-LoRA~\citep{ARDLoRA2025}; the gap to nonlinear practice is bridged empirically in \cref{sec:exp_glue}, with NTK~\citep{Jacot2018NTK,Arora2019NTK} or local-linearization~\citep{LeeWideNetworks2019,Malladi2023LoRANTK} extensions deferred (\cref{sec:conclusion}). Proofs are in \cref{app:proofs}; recall $G_{k,\ell} \leq D_{k,\ell} + B_{k,\ell}$ from \cref{sec:error_decomp}.

\textbf{Main Convergence Result.}
After a component-dependent warm-up, every component's total error decays nearly-linearly (exponential decay modulated by a mild linear prefactor).

\begin{theorem}[Convergence of Parallel Rank-1 Deflation]\label{thm:parallel_convergence}
Under \cref{asump:contraction,asump:gap,asump:bounded_recovery}, where $\calF_k \in (0,1)$ is the per-component contraction factor of \cref{asump:contraction} and $\mathcal{R}_k > 0$ is the magnitude bound of \cref{lem:bounded_recovery}, define effective rates $m_1 = \calF_1$ and
\begin{equation}\label{eq:mk_recurrence}
m_k = \max\Bigl\{\calF_k,  \frac{1}{k} + \frac{k-1}{k} m_{k-1}\Bigr\} \quad \text{for } k \geq 2,
\end{equation}
and starting rounds $s_1 = 1$ with $s_{k+1}$ chosen per-predecessor (i.e., as a maximum over $k' \leq k$) so that
\begin{equation}\label{eq:sk_inline}
s_{k+1} \;=\; \max_{k'\leq k}\, s_{k'} \;+\; \widetilde\calO\!\left(\dfrac{\log\bigl(k\,C_{k+1}\mathcal{R}_{k'}/\calT_{k+1}^\star\bigr)}{1-m_k}\right) \;+\; \dfrac{(k+1)\,m_k}{1-m_k}
\end{equation}
(full Lambert-$W$ form and proof in \cref{app:proof_main}, \cref{cor:sk_simplified}).
Then for all $k \in [r]$ and $\ell \geq s_k - 1$:
\begin{equation}\label{eq:main_bound}
G_{k,\ell} \leq 3\mathcal{R}_k(\ell - s_k + 2)  m_k^{\ell - s_k + 1}.
\end{equation}
\end{theorem}

\begin{remark}[Interpretation of~\eqref{eq:main_bound}]\label{rem:main_bound_meaning}
Each $m_k \in (0,1)$ (\cref{app:mk_below_one}), so the geometric factor dominates the linear prefactor and the bound vanishes exponentially. Proof by strong induction on $k$ (\cref{app:proof_main}); empirical validation in \cref{fig:appendix-convergence,fig:appendix-bounds,fig:appendix-gap}.
\end{remark}

\textbf{Self-Correction.}
The defining advantage of parallel over sequential deflation: errors are transient, not permanent.

\begin{proposition}[Self-Correction]\label{prop:self_correction}
Under the conditions of \cref{thm:parallel_convergence}, for all $\ell \geq s_k$:
\begin{equation}
\norm{\bfY_{k,\ell} - \bfY_k^\star}_F \leq 3\sum_{k'=1}^{k-1}\mathcal{R}_{k'}(\ell - s_{k'} + 1)  m_{k'}^{\ell - s_{k'}} \xrightarrow{\ell \to \infty} 0.
\end{equation}
In sequential deflation, $\norm{\bfY_k - \bfY_k^\star}_F$ is fixed after step $k$ and generically nonzero.
\end{proposition}

The proof bounds $\norm{\bfY_{k,\ell} - \bfY_k^\star}_F \leq \sum_{k'<k} G_{k',\ell-1}$ via \cref{lem:Y_mismatch} and applies \cref{thm:parallel_convergence} to force each $G_{k',\ell-1}\to 0$ (\cref{app:proofs}). \Cref{tab:comparison} summarizes the differences; see \cref{fig:appendix-self-correction}.

\begin{table}[t]
\centering
\caption{Sequential vs.\ parallel rank-1 deflation. ``Front-loaded compute'' means a disproportionate share of iterations $T_k$ must be spent on the first few components, since their errors propagate to all later ones; under parallel deflation a uniform per-component budget suffices because predecessor errors are corrected over rounds.}\label{tab:comparison}
\small
\begin{tabular}{lcc}
\toprule
\textbf{Property} & \textbf{Sequential}~\citep{VandchaliOneRank2024} & \textbf{Parallel (ours)} \\
\midrule
Wall-clock time & $\sum_{k=1}^r T_k$ & $LT$ (up to $r\times$ faster) \\
Deflation mismatch & Fixed, generally nonzero & Vanishes as $\ell \to \infty$ \\
Error propagation & Accumulates and persists & Self-corrects over rounds \\
Compute allocation & Front-loaded compute essential & Uniform per-component budget suffices \\
Estimation error & Fixed algorithmic term & Vanishes as $L \to \infty$ \\
\bottomrule
\end{tabular}
\end{table}

\textbf{Estimation Bounds.}
The self-correction advantage propagates to weight estimation.

\begin{theorem}[Noiseless estimation]\label{thm:gen_noiseless_parallel}
If $\bfY = \bfW^\star\bfX$ with $\mathrm{rank}(\bfW^\star) = r^\star$ and $\svmin(\bfX) > 0$, then after $L$ rounds:
\begin{equation}
\norm{\bfW^\star - \textstyle\sum_{k=1}^r \bfb_{k,L}\bfa_{k,L}^\top}_F
\leq \frac{1}{\svmin(\bfX)}\Bigl(\sum_{k=r+1}^{r^\star} \sigma_k^\star + \sum_{k=1}^r G_{k,L}\Bigr),
\end{equation}
where $\sigma_k^\star$ denotes the $k$-th singular value of $\bfY = \bfW^\star\bfX$.
\end{theorem}

The second term vanishes as $L \to \infty$ by \cref{thm:parallel_convergence}; in the sequential baseline~\citep{VandchaliOneRank2024} the corresponding term is permanently positive.

\begin{theorem}[Noisy estimation]\label{thm:gen_noisy_parallel}
Under $\calE_{ij} \sim \mathcal{N}(0, \varepsilon^2)$ with $\varepsilon \leq \calO(\calT_{\min}^\star / (\sqrt{n} + \sqrt{\log(1/\gamma)}))$, with probability $\geq 1 - \gamma$:
\begin{equation}
\begin{aligned}
\norm{\bfW^\star - \textstyle\sum_k \bfb_{k,L}\bfa_{k,L}^\top}_F
&\leq \underbrace{\textcolor{teal}{\kappa(\bfX)\Bigl(\textstyle\sum_{k>r}\sigma_k(\bfW^\star) + \svmax(\bfX)^{-1}\sum_k G_{k,L}\Bigr)}}_{\text{algorithmic (vanishes)}} \\
&\quad + \underbrace{\textcolor{orange}{\calO\Bigl(\frac{\varepsilon\sqrt{n\log(1/\gamma)}}{\svmin(\bfX)}\Bigl(r + \sqrt{\frac{r}{\calT_{\min}^\star}}\Bigr)\Bigr)}}_{\text{statistical (irreducible)}}.
\end{aligned}
\end{equation}
\end{theorem}

The algorithmic term vanishes via $\sum_k G_{k,L}\to 0$ (\cref{thm:parallel_convergence}); the statistical floor matches the Gaussian-perturbation bound for deflation-based recovery~\citep[Thm.~3]{VandchaliOneRank2024}, so parallelization adds no new noise penalty (the sequential analogue's algorithmic term stays positive; here it decays in $L$; floor matches \cref{fig:appendix-noise}). Advance learning tightens the leading constant $3\mathcal{R}_k$ in~\eqref{eq:main_bound} via $s_k-1$ rounds of private refinement.

\section{Experiments}\label{sec:experiments}

We evaluate \adapad{} along three axes.
\Cref{sec:exp_glue} compares it against four adaptive-rank LoRA baselines on GLUE classification with DeBERTaV3-base.
\Cref{sec:exp_squad} measures its parameter efficiency against fixed-rank LoRA on Qwen3-0.6B SQuAD/SQuAD~v2 reading comprehension.
\Cref{sec:exp_qwen_scaling} measures multi-GPU wall-clock scaling of \emph{component-parallel sharding} (assigning disjoint subsets of \adapad{}'s rank-1 components to different workers, each holding a private copy of the backbone) on the same Qwen3-0.6B + SQuAD workload.
All experiments run on $4{\times}$H200 NVLink; settings in \cref{app:experiments}.

\subsection{GLUE Classification}\label{sec:exp_glue}

\textbf{Setup.}
We fine-tune DeBERTaV3-base on the eight GLUE tasks~\citep{Wang2019GLUE} and compare \adapad{} against fixed-rank LoRA and three adaptive-rank baselines:
(i) \textbf{LoRA} ($r{=}2$);
(ii) \textbf{\adapad{}} ($r_{\max}{=}4$, dynamic) at total rank budget $B{=}144$, matching IncreLoRA's parameter budget;
(iii) \textbf{IncreLoRA}~\citep{ZhangIncreLoRA2023}, reproduced with the authors' code and published hyperparameters;
(iv) \textbf{AdaLoRA}~\citep{ZhangAdaLoRA2023}, \textbf{SoRA}~\citep{DingSoRA2023}, and \textbf{dEBORA}~\citep{dEBORA2025} from their published numbers.
Parameter budgets are approximately matched (0.32--0.47\,M trainable parameters across methods); rank-1 adapters attach to all 72 encoder linear layers (six per layer $\times$ 12 layers: $W_q$, $W_k$, $W_v$, $W_o$, $W_{f1}$, $W_{f2}$), matching IncreLoRA. In practice we use $Q=\infty$ (no projection) as the simplest design choice and find this is sufficient for empirical convergence in all our experiments.

\textbf{Main results.}
\adapad{} achieves the highest GLUE average ($89.34$) in \cref{tab:glue_main}, ahead of AdaLoRA ($89.03$) and our IncreLoRA reproduction ($88.99$).
On the five small- and medium-scale tasks (CoLA, RTE, MRPC, STS-B, SST-2; ${\le}\,67$k training examples) \adapad{} wins four outright and loses RTE by $1.08$ points, three examples on the $276$-example dev set.
On the three large-scale tasks (QNLI, QQP, MNLI; ${\ge}\,105$k training examples), \adapad{} stays within $0.6$ points of the best baseline (each below $0.6\%$ relative loss): $-0.13$ on QNLI, $-0.51$ on QQP, $-0.29$ on MNLI.

\begin{table}[t]
\centering
\caption{GLUE development-set results with DeBERTaV3-base at matched parameter budgets.
\adapad{} discovers rank dynamically ($r_{\max}{=}4$).
Published results are cited from the respective papers; $^\ast$denotes our reproduction with the authors' code.
\textbf{Bold}: best in column. \emph{Avg.}\ is the unweighted mean across the eight tasks.}\label{tab:glue_main}
\footnotesize
\setlength{\tabcolsep}{4pt}
\begin{tabular}{lcccccccccc}
\toprule
Method & \#Params & CoLA & RTE & MRPC & STS-B & SST-2 & QNLI & QQP & MNLI & Avg. \\
 & & (MCC) & (Acc) & (Acc) & (Pear.) & (Acc) & (Acc) & (Acc) & (Acc) & \\
\midrule
LoRA ($r{=}2$)     & 0.34M & 63.56          & 79.06          & 78.86          & 90.47          & 94.95          & 94.03          & 91.61          & 90.30          & 85.36 \\
\textbf{\adapad{}} & 0.34M & \textbf{72.26} & 86.28          & \textbf{91.67} & \textbf{91.86} & \textbf{96.33} & 94.36          & 91.55          & 90.37          & \textbf{89.34} \\
\midrule
IncreLoRA$^\ast$   & 0.34M & 71.85          & 85.92          & 90.20          & 91.68          & 96.22          & 93.50          & 91.94          & 90.59          & 88.99 \\
AdaLoRA            & 0.32M & 70.04          & \textbf{87.36} & 90.44          & 91.63          & 95.80          & \textbf{94.49} & 91.78          & \textbf{90.66} & 89.03 \\
SoRA ($r_{\max}{=}4$) & 0.47M & 71.05 & 86.04 & 90.20 & 91.76 & 95.57 & 93.92 & \textbf{92.06} & 90.38 & 88.87 \\
dEBORA             & 0.40M & 68.72          & 83.75          & 90.16          & 90.84          & 95.29          & 93.42          & 91.88          & 90.01          & 88.01 \\
\bottomrule
\end{tabular}
\end{table}

\textbf{Discovered rank distribution.}
\Cref{fig:rank_discovery} shows per-module rank on CoLA at convergence: of 72 modules, 16/44/8/4 sit at rank $1/2/3/4$. FFN intermediate $W_{f1}$ averages $2.67$ and attention output $W_o$ averages $2.17$, while $W_{f2}$ stays lowest ($1.58$); layer 11 averages $3.67$, consistent with last-layer task-specificity~\citep{ZhangAdaLoRA2023}. Forcing fixed rank costs $8.70$ MCC (\cref{tab:ablation}).

\begin{figure}[ht]
\centering
\includegraphics[width=0.85\textwidth]{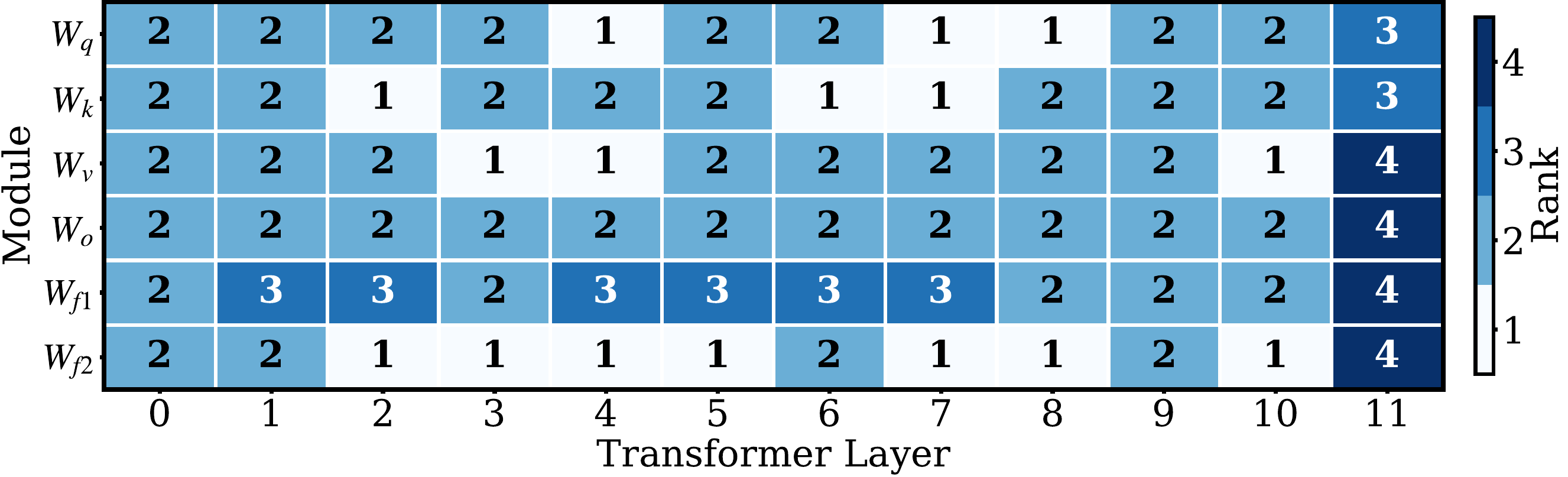}
\caption{Per-module rank discovered by \adapad{} on CoLA (avg $=2.0$); 72 modules ($6$ types $\times$ $12$ layers). Higher rank concentrates in the last layer and FFN intermediate projections.}\label{fig:rank_discovery}
\end{figure}

\textbf{Ablation study.}
\Cref{tab:ablation} isolates \adapad{}'s three ingredients on CoLA and MRPC: dynamic rank discovery is the largest contributor ($-8.70$ MCC, $-12.81$ Acc when replaced by fixed $r{=}2$), then importance scoring ($-2.69$, $-1.72$ vs.\ uniform), then advance learning ($-1.21$, $-1.23$).

\begin{table}[t]
\centering
\caption{Ablation on CoLA and MRPC.
Each row removes one component from the full \adapad{} algorithm.
$\Delta$: change from the full model.}\label{tab:ablation}
\small
\begin{tabular}{lcccccc}
\toprule
 & \multicolumn{3}{c}{CoLA (MCC)} & \multicolumn{3}{c}{MRPC (Acc)} \\
\cmidrule(lr){2-4} \cmidrule(lr){5-7}
Variant & Score & $\Delta$ & & Score & $\Delta$ & \\
\midrule
\adapad{} & \textbf{72.26} & --- & & \textbf{91.67} & --- & \\
\quad w/o advance learning & 71.05 & $-$1.21 & & 90.44 & $-$1.23 & \\
\quad w/o dynamic rank (fixed $r{=}2$) & 63.56 & $-$8.70 & & 78.86 & $-$12.81 & \\
\quad w/o importance scoring (uniform) & 69.57 & $-$2.69 & & 89.95 & $-$1.72 & \\
\bottomrule
\end{tabular}
\end{table}

\subsection{Parameter Efficiency on Reading Comprehension}\label{sec:exp_squad}

We test \adapad{} on Qwen3-0.6B fine-tuned for QA on SQuAD~v1.1 and SQuAD~v2~\citep{Rajpurkar2016SQuAD,Rajpurkar2018SQuADv2} (generative decoder, adapters on seven projections $W_q,W_k,W_v,W_o,W_{\text{gate}},W_{\text{up}},W_{\text{down}}$; 196 modules). LoRA uses uniform rank $r\in\{4,6\}$; \adapad{} uses average-rank budget $\bar r\in\{4,6\}$ with $r_{\max}{=}8$ (full settings in \cref{app:squad_experiments}). For each cell we report the smallest \adapad{} adapter that matches or beats same-budget LoRA on F1, with parameter reduction $\Delta\text{Par}$.

\begin{table}[t]
\centering
\caption{Matched-budget head-to-head on Qwen3-0.6B SQuAD/SQuAD~v2 (seed=43; bf16; identical recipe across all cells; full settings in \cref{app:squad_experiments}).
\textbf{Bold}: \adapad{} matches or beats LoRA at the cell.}\label{tab:causal_lm_main}
\small
\setlength{\tabcolsep}{4pt}
\begin{tabular}{llcccccc>{\bfseries}c}
\toprule
& & \multicolumn{3}{c}{LoRA (fixed-rank $r$)} & \multicolumn{3}{c}{\adapad{} ($\bar r$)} & \\
\cmidrule(lr){3-5}\cmidrule(lr){6-8}
Task & $b$ & Params & F1 & EM & Params & F1 & EM & $\Delta$Par \\
\midrule
\multirow{2}{*}{SQuAD v1.1}
  & 4 & 2.52M & 88.48 & 80.65 & 2.12M & \textbf{88.83} & \textbf{80.77} & $-15.9\%$ \\
  & 6 & 3.79M & 88.93 & 81.42 & 2.11M & \textbf{89.02} & 81.09           & $-44.3\%$ \\
\midrule
\multirow{2}{*}{SQuAD v2}
  & 4 & 2.52M & 73.45 & 70.33 & 1.93M & \textbf{73.63} & \textbf{70.70} & $-23.4\%$ \\
  & 6 & 3.79M & 73.91 & 70.66 & 2.30M & \textbf{73.99} & \textbf{70.91} & $-39.3\%$ \\
\bottomrule
\end{tabular}
\end{table}

Across the four cells (\cref{tab:causal_lm_main}), \adapad{} yields adapters $15.9$--$44.3\%$ smaller than fixed-rank LoRA (mean $30.7\%$) while remaining competitive on F1/EM; a complementary view to GLUE, which fixes the budget and reads off quality.

\subsection{Multi-GPU Scaling on Qwen3-0.6B}\label{sec:exp_qwen_scaling}
\adapad{} admits component-parallel sharding: each worker holds a disjoint subset of the rank-1 components and trains them independently between syncs, predicting near-linear speedup. We measure this on the SQuAD~v1.1 workload of \cref{sec:exp_squad}, disabling per-module growth and pinning every module at $r_{\max}$ to fix the total at $196 \times r_{\max}$ across $P\in\{1,2,3,4\}$ at each $r_{\max}\in\{4,8\}$. We report mean${}\pm{}$std over five global rounds (hyperparameters in \cref{app:squad_experiments}).

\begin{table}[t]
\centering
\caption{Component-parallel scaling on Qwen3-0.6B + SQuAD~v1.1.
Per-batch wall-clock cost (mean${}\pm{}$std over five global rounds) and speedup vs.\ $P{=}1$ at the same rank.
\textbf{Bold}: best speedup in each rank column.}\label{tab:qwen_scaling}
\small
\setlength{\tabcolsep}{4pt}
\begin{tabular}{lcccc}
\toprule
& \multicolumn{2}{c}{$r_{\max}{=}4$} & \multicolumn{2}{c}{$r_{\max}{=}8$} \\
\cmidrule(lr){2-3} \cmidrule(lr){4-5}
\#GPUs & per-batch (s) & speedup & per-batch (s) & speedup \\
\midrule
1 & $0.986${\scriptsize${}\pm 0.011$} & $1.00{\times}$                   & $2.039${\scriptsize${}\pm 0.015$} & $1.00{\times}$ \\
2 & $0.532${\scriptsize${}\pm 0.006$} & $1.85${\scriptsize${}\pm 0.03$}${\times}$         & $1.062${\scriptsize${}\pm 0.008$} & $1.92${\scriptsize${}\pm 0.02$}${\times}$ \\
3 & $0.487${\scriptsize${}\pm 0.006$} & $2.03${\scriptsize${}\pm 0.03$}${\times}$         & $0.806${\scriptsize${}\pm 0.009$} & $2.53${\scriptsize${}\pm 0.04$}${\times}$ \\
4 & $0.403${\scriptsize${}\pm 0.047$} & $\mathbf{2.45}${\scriptsize${}\pm 0.29$}$\mathbf{{\times}}$ & $0.768${\scriptsize${}\pm 0.074$} & $\mathbf{2.66}${\scriptsize${}\pm 0.26$}$\mathbf{{\times}}$ \\
\bottomrule
\end{tabular}
\end{table}

The 1-3 GPU cells are highly stable (rel.\ std $<1.5\%$): three workers deliver $2.03{\times}$ at $r_{\max}{=}4$ and $2.53{\times}$ at $r_{\max}{=}8$. At $P{=}4$ we see $2.45{\times}$ and $2.66{\times}$ with higher variance (rel.\ std ${\approx}10\%$), as the 4-way NVLink all-reduce makes NCCL latency dominate when per-worker compute shrinks; overlap-friendly schedules~\citep{LiPyTorchDDP2020,SergeevHorovod2018} and bandwidth-aware placement~\citep{RajbhandariZeRO2020,ShoeybiMegatron2019} are future work.

\section{Discussion and Conclusion}\label{sec:conclusion}

\adapad{} addresses two LoRA design questions: parallel training keeps early rank-1 components from corrupting later ones, and per-module dynamic rank discovery makes the rank distribution an output of importance scores rather than a fixed input.
We prove nearly-linear convergence with a self-correcting deflation mismatch and an estimation bound that splits into a vanishing algorithmic term plus an irreducible statistical floor.
Empirically, \adapad{} achieves the highest GLUE average against four adaptive-rank baselines (\cref{sec:exp_glue}), produces Qwen3-0.6B SQuAD/SQuAD~v2 adapters $30.7\%$ smaller than fixed-rank LoRA at competitive F1/EM (\cref{sec:exp_squad}), and component-parallel sharding delivers up to $2.66{\times}$ speedup on $4{\times}$H200 (\cref{sec:exp_qwen_scaling}).

\textbf{Limitations and future work.}
The theory covers the bilinear model~\eqref{eq:main_problem}; extension to nonlinear LoRA losses is empirical, bounds degrade as $\calT_k^\star \to 0$, and synchronous $\calO(r^2(d+m))$ communication can offset the speedup at very large $r$. Validation at $7$B+ scale, NTK~\citep{Jacot2018NTK,Arora2019NTK} or local-linearization~\citep{LeeWideNetworks2019,Malladi2023LoRANTK} extensions, asynchronous communication, and richer importance scores are future work.

\bibliographystyle{plainnat}
\bibliography{ref}

\newpage
\appendix

\section{Rank-1 Subroutines}\label{app:subroutines}

We provide pseudocode for two rank-1 subroutines that satisfy \cref{asump:contraction}.

\begin{algorithm}[ht]
\caption{$\topone$ via Alternating Least Squares}\label{alg:rank1_als}
\begin{algorithmic}[1]
\Require $\bfY_k \in \R^{m \times n}$, $\bfX \in \R^{d \times n}$, iterations $T$, warm-start $(\bfa^{(0)}, \bfb^{(0)})$
\State $\bfM \gets \bfX\bfX^\top$; $\bfN \gets \bfY_k\bfX^\top$
\For{$t = 0, \ldots, T{-}1$}
  \State $\bfb^{(t+1)} \gets \bfN\bfa^{(t)} / (\bfa^{(t)\top}\bfM\bfa^{(t)})$
  \State $\bfa^{(t+1)} \gets (\|\bfb^{(t+1)}\|^2\bfM)^{-1}\bfN^\top\bfb^{(t+1)}$
\EndFor
\State \Return $(\bfa^{(T)}, \bfb^{(T)})$
\end{algorithmic}
\end{algorithm}

\begin{algorithm}[ht]
\caption{$\topone$ via Factored Gradient Descent}\label{alg:rank1_gd}
\begin{algorithmic}[1]
\Require $\bfY_k$, $\bfX$, $T$, warm-start $(\bfa^{(0)}, \bfb^{(0)})$, learning rates $\eta_\bfa, \eta_\bfb$
\State $\bfM \gets \bfX\bfX^\top$; $\bfN \gets \bfY_k\bfX^\top$
\For{$t = 0, \ldots, T{-}1$}
  \State $\nabla_\bfa \gets \bfM\bfa^{(t)}\|\bfb^{(t)}\|^2 - \bfN^\top\bfb^{(t)}$
  \State $\nabla_\bfb \gets \bfb^{(t)}(\bfa^{(t)\top}\bfM\bfa^{(t)}) - \bfN\bfa^{(t)}$
  \State $\bfa^{(t+1)} \gets \bfa^{(t)} - \eta_\bfa\nabla_\bfa$; $\bfb^{(t+1)} \gets \bfb^{(t)} - \eta_\bfb\nabla_\bfb$
\EndFor
\State \Return $(\bfa^{(T)}, \bfb^{(T)})$
\end{algorithmic}
\end{algorithm}

\section{Background Perturbation Results}\label{app:background}

\begin{theorem}[Weyl's Inequality {\citep{Weyl1912}}]\label{thm:weyl}
For $\bfM, \bfDelta \in \R^{m \times n}$, define $\tilde{\bfM} = \bfM + \bfDelta$.
Then $|\tilde{\sigma}_i - \sigma_i| \leq \|\bfDelta\|_2$ for all $i$.
\end{theorem}

\begin{theorem}[Wedin's Theorem {\citep{Wedin1972}}]\label{thm:wedin}
Let $\bfM$ and $\tilde{\bfM} = \bfM + \bfDelta$ have top singular triplets $(\sigma_1, \bfu_1, \bfv_1)$ and $(\tilde{\sigma}_1, \tilde{\bfu}_1, \tilde{\bfv}_1)$.
If $\delta := \min\{\min_{j \neq 1}|\sigma_1 - \tilde{\sigma}_j|, \sigma_1\} > 0$, then
$\sin^2\angle(\bfu_1, \tilde{\bfu}_1) + \sin^2\angle(\bfv_1, \tilde{\bfv}_1) \leq (\|\bfDelta^\top \bfu_1\|_2^2 + \|\bfDelta \bfv_1\|_2^2) / \delta^2$.
\end{theorem}

\section{Notation Index}\label{app:notation}

\begin{table}[ht]
\centering
\caption{Summary of notation.}\label{tab:notation_index}
\small
\begin{tabular}{llll}
\toprule
\textbf{Symbol} & \textbf{Type} & \textbf{Defined in} & \textbf{Meaning} \\
\midrule
$\bfa_{k,\ell}$ & $\R^d$ & \cref{alg:parallel_rank1} & Worker $k$'s right factor at round $\ell$ \\
$\bfa_k^\star$ & $\R^d$ & \cref{sec:setup} & Ideal right factor for component $k$ \\
$B_{k,\ell}$ & $\R_{\geq 0}$ & \cref{sec:error_decomp}, \eqref{eq:def_B} & Deflation mismatch $\norm{\bar\bfb_{k,\ell}\bar\bfa_{k,\ell}^\top\bfX - \bfb_k^\star\bfa_k^{\star\top}\bfX}_F$ \\
$\bfb_{k,\ell}$ & $\R^m$ & \cref{alg:parallel_rank1} & Worker $k$'s left factor at round $\ell$ \\
$C_k$ & $\R_{>0}$ & \cref{lem:B_bound} & $3\sigma_k^\star/\calT_k^\star + 1$ \\
$D_{k,\ell}$ & $\R_{\geq 0}$ & \cref{sec:error_decomp}, \eqref{eq:def_D} & Numerical error $\norm{\bfb_{k,\ell}\bfa_{k,\ell}^\top\bfX - \bar\bfb_{k,\ell}\bar\bfa_{k,\ell}^\top\bfX}_F$ \\
$G_{k,\ell}$ & $\R_{\geq 0}$ & \cref{sec:error_decomp}, \eqref{eq:def_G} & Total error $\norm{\bfb_{k,\ell}\bfa_{k,\ell}^\top\bfX - \bfb_k^\star\bfa_k^{\star\top}\bfX}_F$ \\
$\calF_k$ & $(0,1)$ & \cref{asump:contraction} & Subroutine contraction factor \\
$m_k$ & $(0,1)$ & \cref{thm:parallel_convergence} & Effective convergence rate \\
$\mathcal{R}_k$ & $\R_{>0}$ & \hl{\cref{lem:bounded_recovery}} & Magnitude bound on the rank-1 recovery \\
$s_k$ & $\N$ & \cref{thm:parallel_convergence} & Warm-up round \\
$\hat{s}_k$ & $\N$ & \cref{lem:surrogate} & Earliest round at which the Wedin gap activates \\
$\calT_k^\star$ & $\R_{>0}$ & \cref{asump:gap} & Singular-value gap \\
$\What$ & -- & \cref{thm:parallel_convergence} & $\max\{1, -W_{-1}(-\cdot)\}$ \\
\bottomrule
\end{tabular}
\end{table}

\section{Complete Proofs}\label{app:proofs}
Recall the following definitions:
\begin{align}
D_{k,\ell} & \triangleq  \norm{\bfb_{k,\ell}\bfa_{k,\ell}^\top\bfX - \bar{\bfb}_{k,\ell}\bar{\bfa}_{k,\ell}^\top\bfX}_F, \quad \quad \text{(Numerical error)} \\
B_{k,\ell} & \triangleq  \norm{\bar{\bfb}_{k,\ell}\bar{\bfa}_{k,\ell}^\top\bfX - \bfb_k^\star\bfa_k^{\star\top}\bfX}_F,  \quad ~~\quad \text{(Deflation mismatch)} \\
G_{k,\ell} & \triangleq  \norm{\bfb_{k,\ell}\bfa_{k,\ell}^\top\bfX - \bfb_k^\star\bfa_k^{\star\top}\bfX}_F. \quad \quad ~~~\text{(Total error)}
\end{align}
where the ideal fit $(\bar{\bfa}_{k,\ell}, \bar{\bfb}_{k,\ell})$ is defined in~\eqref{eq:ideal_fit}; and $(\bfa_k^\star, \bfb_k^\star)$ denotes the ideal $k$-th component of the clean target $\bfY_k^\star$.
\subsection{Supporting Lemmas}
The following lemmas control how these three quantities interact across rounds.

\hl{[Moved from background.tex per Jasper:]}
\begin{lemma}[Bounded Rank-1 Recovery]\label{lem:bounded_recovery}
Under \cref{asump:bounded_recovery}, the per-component recovery error satisfies
\begin{equation}\label{eq:bounded_recovery}
G_{k,\ell} \;\leq\; \mathcal{R}_k \;:=\; Q + \sigma_k^\star
\qquad \text{for every round } \ell \geq 0,
\end{equation}
and the same bound transfers to $D_{k,\ell}$ and $B_{k,\ell}$.
\end{lemma}
\begin{proof}
By the projection in \cref{line:project}, $\norm{\bfb_{k,\ell}\bfa_{k,\ell}^\top\bfX}_F \leq Q$, and the ideal $k$-th component satisfies $\norm{\bfb_k^\star\bfa_k^{\star\top}\bfX}_F = \sigma_k^\star$. The triangle inequality gives
\[
G_{k,\ell} \;\leq\; \norm{\bfb_{k,\ell}\bfa_{k,\ell}^\top\bfX}_F + \norm{\bfb_k^\star\bfa_k^{\star\top}\bfX}_F \;\leq\; Q + \sigma_k^\star \;=\; \mathcal{R}_k.
\]
The corresponding ceilings for $D_{k,\ell}$ and $B_{k,\ell}$ follow by the same triangle-inequality argument using the ideal-fit norm $\norm{\bar{\bfb}_{k,\ell}\bar{\bfa}_{k,\ell}^\top\bfX}_F = \sigma_{1_k,\ell} \leq \norm{\bfY_{k,\ell}}_2$, which is bounded by $Q$ whenever $\norm{\bfY_{k,\ell}-\bfY_k^\star}_2 \leq \sigma_1^\star$ (using $Q \geq 2\sigma_1^\star$); this is the regime in which the proofs invoke these ceilings.
\end{proof}

\begin{lemma}[Numerical Error Recurrence]\label{lem:D_recurrence}
Under \cref{asump:contraction} and \cref{asump:bounded_recovery}, for all $\ell \geq k$ if $\norm{\bfY_{k,\ell} - \bfY_{k}^\star}\leq \sigma_k^\star$, then we have that
$D_{k,\ell} \leq \calF_k(D_{k,\ell-1} + B_{k,\ell} + B_{k,\ell-1})$.
\end{lemma}
\begin{proof}
Let the output of Line 9 be denoted by $\tilde{\bfa}_{k,\ell},\tilde{\bfb}_{k,\ell}$. If $\norm{\tilde{\bfb}_{k,\ell}\tilde{\bfa}_{k,\ell}^\top\bfX}\leq Q$, then $\tilde{\bfa}_{k,\ell} = \bfa_{k,\ell},\tilde{\bfb}_{k,\ell} = \bfb_{k,\ell}$. Therefore, by \cref{asump:contraction}:
\begin{align*}
D_{k,\ell}
&\leq \calF_k \norm{\bfb_{k,\ell-1}\bfa_{k,\ell-1}^\top\bfX - \bar{\bfb}_{k,\ell}\bar{\bfa}_{k,\ell}^\top\bfX}_F \\
&\leq \calF_k(\norm{\bfb_{k,\ell-1}\bfa_{k,\ell-1}^\top\bfX - \bfb_k^\star\bfa_k^{\star\top}\bfX}_F + \norm{\bfb_k^\star\bfa_k^{\star\top}\bfX - \bar{\bfb}_{k,\ell}\bar{\bfa}_{k,\ell}^\top\bfX}_F) \\
&= \calF_k(G_{k,\ell-1} + B_{k,\ell}) \leq \calF_k(D_{k,\ell-1} + B_{k,\ell-1} + B_{k,\ell}).
\end{align*}
On the other hand, if $\norm{\tilde{\bfb}_{k,\ell}\tilde{\bfa}_{k,\ell}^\top\bfX}\geq Q$, we then notice that
\[
    \norm{\bfY_{k,\ell}}_2 \leq \norm{\bfY_k^\star}_2 + \norm{\bfY_{k,\ell} - \bfY_k^\star}_2 \leq \sigma_k^\star + \sigma_1^\star \leq Q
\]
Therefore, $\norm{\bar{\bfb}_{k,\ell}\bar{\bfa}_{k,\ell}^\top\bfX}_F \leq Q$. By the property of projection onto convex sets, we have that 
\[
    \norm{\bfb_{k,\ell}\bfa_{k,\ell}^\top\bfX - \bar{\bfb}_{k,\ell}\bar{\bfa}_{k,\ell}^\top\bfX} \leq \norm{\tilde{\bfb}_{k,\ell}\tilde{\bfa}_{k,\ell}^\top\bfX - \bar{\bfb}_{k,\ell}\bar{\bfa}_{k,\ell}^\top\bfX}
\]
Therefore, by \cref{asump:contraction} we have that
\begin{align*}
D_{k,\ell} & \leq \norm{\tilde{\bfb}_{k,\ell}\tilde{\bfa}_{k,\ell}^\top\bfX - \bar{\bfb}_{k,\ell}\bar{\bfa}_{k,\ell}^\top\bfX}\\
&\leq \calF_k \norm{\bfb_{k,\ell-1}\bfa_{k,\ell-1}^\top\bfX - \bar{\bfb}_{k,\ell}\bar{\bfa}_{k,\ell}^\top\bfX}_F \\
&\leq \calF_k(\norm{\bfb_{k,\ell-1}\bfa_{k,\ell-1}^\top\bfX - \bfb_k^\star\bfa_k^{\star\top}\bfX}_F + \norm{\bfb_k^\star\bfa_k^{\star\top}\bfX - \bar{\bfb}_{k,\ell}\bar{\bfa}_{k,\ell}^\top\bfX}_F) \\
&= \calF_k(G_{k,\ell-1} + B_{k,\ell}) \leq \calF_k(D_{k,\ell-1} + B_{k,\ell-1} + B_{k,\ell}).
\end{align*}
\end{proof}

\begin{lemma}[Deflated Target Mismatch]\label{lem:Y_mismatch}
For \cref{alg:parallel_rank1}: $\norm{\bfY_{k,\ell} - \bfY_k^\star}_F \leq \sum_{k'=1}^{k-1} G_{k',\ell-1}$.
\end{lemma}
\begin{proof}
$\bfY_{k,\ell} - \bfY_k^\star = \sum_{k'<k}(\bfb_{k'}^\star\bfa_{k'}^{\star\top} - \bfb_{k',\ell-1}\bfa_{k',\ell-1}^\top)\bfX$.
Taking Frobenius norms: $\norm{\bfY_{k,\ell} - \bfY_k^\star}_F \leq \sum_{k'<k} G_{k',\ell-1}$.
\end{proof}

\begin{lemma}[Deflation Mismatch Bound]\label{lem:B_bound}
Under \cref{asump:gap}, define $C_k := 3\sigma_k^\star/\calT_k^\star + 1$.
If $\|\bfY_{k,\ell} - \bfY_k^\star\|_2 < \frac{1}{2}\min_{j>k}|\sigma_k^\star - \sigma_j^\star|$, then $B_{k,\ell} \leq C_k\norm{\bfY_{k,\ell} - \bfY_k^\star}_F$.
\end{lemma}

\begin{proof}
The deflation mismatch $B_{k,\ell}$ measures the gap between the best rank-1 fit of the \emph{actual} target $\bfY_{k,\ell}$ and the ideal $k$-th component of the \emph{clean} target $\bfY_k^\star$.
We bound this gap by decomposing it into a singular-vector error and a singular-value error, then control each via the two perturbation results in \cref{app:background}: Wedin's theorem for the singular vectors and Weyl's inequality for the singular values.

\emph{Step 1 (Rank-1 SVD expressions).}
By the Eckart-Young-Mirsky theorem, ideal sequential deflation removes exactly the first $k-1$ singular components of $\bfY$, so $\bfY_k^\star = \sum_{j \geq k}\sigma_j^\star\bfu_j^\star\bfv_j^{\star\top}$ has leading singular triplet $(\sigma_k^\star, \bfu_k^\star, \bfv_k^\star)$ with gap $\calT_k^\star$~\citep[Lemma~1]{VandchaliOneRank2024}.
The hypothesis $\|\bfY_{k,\ell}-\bfY_k^\star\|_2 < \tfrac{1}{2}\calT_k^\star$ together with \cref{asump:gap} (distinct singular values of $\bfY_k^\star$) ensures that $\bfY_{k,\ell}$ has a unique top singular triplet $(\sigma_{1_k,\ell}, \bfu_{1_k,\ell}, \bfv_{1_k,\ell})$, so the ideal fit~\eqref{eq:ideal_fit} is well-defined.
By definition of the ideal fit and the ideal component,
\begin{equation}\label{eq:svd_expressions}
\bar{\bfb}_{k,\ell}\bar{\bfa}_{k,\ell}^\top\bfX = \sigma_{1_k,\ell}\bfu_{1_k,\ell}\bfv_{1_k,\ell}^\top, \qquad \bfb_k^\star\bfa_k^{\star\top}\bfX = \sigma_k^\star\bfu_k^\star\bfv_k^{\star\top}.
\end{equation}

%\textcolor{orange}{Jasper: the following decomposition of $B_{k,\ell}$ seems unclear. An explicit decomposition of $B_{k,\ell}$ is needed (why into the singular vector difference and the singular value difference.}

\emph{Step 2 (Decomposition into singular-vector and singular-value errors).}
Using~\eqref{eq:svd_expressions} and the definition of $B_{k,\ell}$ in~\eqref{eq:def_B}, we add and subtract $\sigma_k^\star\bfu_{1_k,\ell}\bfv_{1_k,\ell}^\top$ and apply the triangle inequality:
\begin{align}
B_{k,\ell}
&= \norm{\sigma_{1_k,\ell}\bfu_{1_k,\ell}\bfv_{1_k,\ell}^\top - \sigma_k^\star\bfu_k^\star\bfv_k^{\star\top}}_F \nonumber\\
&= \norm{(\sigma_{1_k,\ell} - \sigma_k^\star)\bfu_{1_k,\ell}\bfv_{1_k,\ell}^\top + \sigma_k^\star(\bfu_{1_k,\ell}\bfv_{1_k,\ell}^\top - \bfu_k^\star\bfv_k^{\star\top})}_F \nonumber\\
&\leq \underbrace{|\sigma_k^\star - \sigma_{1_k,\ell}|}_{\text{singular-value error}} + \sigma_k^\star \underbrace{\norm{\bfu_k^\star\bfv_k^{\star\top} - \bfu_{1_k,\ell}\bfv_{1_k,\ell}^\top}_F}_{\text{singular-vector error}}, \label{eq:B_decomp}
\end{align}
using $\norm{\bfu_{1_k,\ell}\bfv_{1_k,\ell}^\top}_F = 1$.
The decomposition is natural because the actual and ideal objects in~\eqref{eq:svd_expressions} are each a scalar times a unit-Frobenius rank-1 term, so the discrepancy decomposes cleanly into a scale mismatch and a direction mismatch.

\emph{Step 3 (Singular-value error via Weyl).}
By \cref{thm:weyl}, $|\sigma_k^\star - \sigma_{1_k,\ell}| \leq \norm{\bfY_k^\star - \bfY_{k,\ell}}_2 \leq \norm{\bfY_k^\star - \bfY_{k,\ell}}_F$.

\emph{Step 4 (Singular-vector error via Wedin).}
Insert $\pm\,\bfu_k^\star\bfv_{1_k,\ell}^\top$ and apply the triangle inequality:
\begin{equation}\label{eq:uv_split}
\norm{\bfu_k^\star\bfv_k^{\star\top} - \bfu_{1_k,\ell}\bfv_{1_k,\ell}^\top}_F
\leq \norm{\bfv_k^\star - \bfv_{1_k,\ell}}_2 + \norm{\bfu_k^\star - \bfu_{1_k,\ell}}_2
\leq \sqrt{2}\sqrt{\sin^2\alpha + \sin^2\beta},
\end{equation}
where $\alpha := \angle(\bfu_k^\star, \bfu_{1_k,\ell})$, $\beta := \angle(\bfv_k^\star, \bfv_{1_k,\ell})$, and we used $\norm{\bfu-\tilde{\bfu}}_2 \leq \sqrt{2}\sin\angle(\bfu,\tilde{\bfu})$ for unit vectors.
To apply Wedin's theorem (\cref{thm:wedin}) we first verify the gap condition: by Weyl's inequality, $|\sigma_j(\bfY_{k,\ell}) - \sigma_j(\bfY_k^\star)| \leq \norm{\bfY_{k,\ell} - \bfY_k^\star}_2$ for every $j$, which combined with the hypothesis $\norm{\bfY_{k,\ell} - \bfY_k^\star}_2 < \tfrac{1}{2}\min_{j>k}|\sigma_k^\star - \sigma_j^\star|$ yields
$\delta_k := \min\{\min_{j\geq 2}|\sigma_1(\bfY_{k,\ell}) - \sigma_j(\bfY_{k,\ell})|, \sigma_1(\bfY_{k,\ell})\} \geq \calT_k^\star/2 > 0$.
\cref{thm:wedin} then gives $\sin^2\alpha + \sin^2\beta \leq 2\norm{\bfY_k^\star - \bfY_{k,\ell}}_F^2/(\calT_k^\star)^2$.
Plugging back into~\eqref{eq:uv_split} and using $2\sqrt{2} < 3$,
\begin{equation}\label{eq:uv_bound}
\norm{\bfu_k^\star\bfv_k^{\star\top} - \bfu_{1_k,\ell}\bfv_{1_k,\ell}^\top}_F \leq \frac{2\sqrt{2}}{\calT_k^\star}\norm{\bfY_k^\star - \bfY_{k,\ell}}_F \leq \frac{3}{\calT_k^\star}\norm{\bfY_k^\star - \bfY_{k,\ell}}_F.
\end{equation}

\emph{Step 5 (Combine).}
Substituting Steps 3 and~\eqref{eq:uv_bound} into~\eqref{eq:B_decomp}:
\begin{equation}
B_{k,\ell} \leq \frac{3\sigma_k^\star}{\calT_k^\star}\norm{\bfY_k^\star - \bfY_{k,\ell}}_F + \norm{\bfY_k^\star - \bfY_{k,\ell}}_F = C_k\norm{\bfY_k^\star - \bfY_{k,\ell}}_F, \quad C_k := \tfrac{3\sigma_k^\star}{\calT_k^\star} + 1.
\end{equation}
\end{proof}

\subsection{Proof of Theorem~\ref{thm:parallel_convergence}}\label{app:proof_main}

For the reader's convenience, we restate the theorem before giving the proof.

\textbf{\cref{thm:parallel_convergence} (Restated).}
\emph{Under \cref{asump:contraction,asump:gap,asump:bounded_recovery}, define $m_1 = \calF_1$ and $m_k = \max\{\calF_k,\; \tfrac{1}{k} + \tfrac{k-1}{k}m_{k-1}\}$ for $k \geq 2$, and let $s_1 = 1$ with $s_{k+1}$ chosen via~\eqref{eq:sk_condition}. Then for all $k \in [r]$ and all $\ell \geq s_k - 1$,
$G_{k,\ell} \leq 3\mathcal{R}_k(\ell - s_k + 2)\,m_k^{\ell - s_k + 1}.$}

\begin{proof}
We proceed by strong induction on the component index $k$.
We first establish a coupled \emph{triple recursion} on the numerical error, deflation mismatch, and total error; then construct upper-bounding \emph{surrogate sequences} that admit closed-form decay; finally, verify the \emph{five warm-up conditions} on $s_k$ that make the surrogates valid. To start, let $s_k$ satisfy the property that $\norm{\bfY_{k,\ell} - \bfY_k^\star}_F \leq \sigma_1^\star$.

\paragraph{Base case ($k=1$).}\label{sec:proof_base}
$\bfY_{1,\ell} = \bfY$ for all $\ell$, so the first deflation target is exact and $B_{1,\ell} = 0$.
Then $G_{1,\ell} = D_{1,\ell} \leq \calF_1^{\ell}D_{1,0} \leq \mathcal{R}_1\calF_1^{\ell}$ by \hl{\cref{asump:contraction} and \cref{lem:bounded_recovery}}.
Since $m_1 = \calF_1$ and $s_1 = 1$, $G_{1,\ell} \leq \mathcal{R}_1 m_1^{\ell} \leq 3\mathcal{R}_1(\ell+1)m_1^{\ell}$, which satisfies~\eqref{eq:main_bound}.
The tighter factor $1$ for the base case is loosened to $3$ once deflation enters at $k\geq 2$, because the surrogate initialization $\hat{G}_{k,s_k-1} = D_{k,s_k-1} + \hat{B}_{k,s_k-1} + \hat{B}_{k,s_k-2} \leq 3\mathcal{R}_k$ via \hl{\cref{lem:bounded_recovery}}.

\paragraph{Inductive step: triple recursion.}
Fix $k\geq 2$ and assume~\eqref{eq:main_bound} holds for every $k'<k$.
We collect the three error recursions used below.

\emph{(i) Numerical error.}
\cref{lem:D_recurrence} gives, for $\ell \geq s_k$,
\begin{equation}\label{eq:triple_D}
D_{k,\ell} \leq \calF_k\bigl(D_{k,\ell-1} + B_{k,\ell} + B_{k,\ell-1}\bigr).
\end{equation}

\emph{(ii) Deflation mismatch.}
Combining \cref{lem:Y_mismatch} with \cref{lem:B_bound},
\begin{equation}\label{eq:triple_B}
B_{k,\ell} \leq C_k\norm{\bfY_{k,\ell}-\bfY_k^\star}_F \leq C_k\sum_{k'<k} G_{k',\ell-1}.
\end{equation}
The key structural observation is that $B_{k,\ell}$ depends only on the \emph{earlier} components' total errors $G_{k',\ell-1}$ with $k'<k$: predecessor convergence automatically forces $B_{k,\ell}\to 0$.

\emph{(iii) Total error.}
Unrolling~\eqref{eq:triple_D} from round $s_k-1$ up to round $\ell$ and using $G_{k,\ell} \leq D_{k,\ell} + B_{k,\ell}$,
\begin{equation}\label{eq:triple_G}
G_{k,\ell} \leq \calF_k^{\ell-s_k+1}D_{k,s_k-1} + \sum_{\ell'=s_k-1}^{\ell-1}\calF_k^{\ell-\ell'}\bigl(B_{k,\ell'}+B_{k,\ell'-1}\bigr) + B_{k,\ell}.
\end{equation}
The three recursions~\eqref{eq:triple_D}--\eqref{eq:triple_G} are the natural triple-recursion structure for parallel deflation, with the mismatch bound supplied by our Wedin-based \cref{lem:B_bound}.

\paragraph{Surrogate sequences.}
Because the right-hand sides of~\eqref{eq:triple_B}--\eqref{eq:triple_G} are coupled through $\{G_{k',\cdot}\}_{k'<k}$, we replace the actual sequences with explicit upper-bounding surrogates.
Let $\hat{s}_k$ denote the earliest round at which the Wedin gap condition in \cref{lem:B_bound} holds; $\hat{s}_k \leq s_k$.
Following \cref{lem:surrogate}, define the \emph{surrogate mismatch} and \emph{surrogate total error} piecewise:
\begin{equation}\label{eq:surrogate_piecewise}
\hat{B}_{k,\ell} = \begin{cases} \mathcal{R}_k & \text{if } \ell < \hat{s}_k, \\ \min\bigl\{\mathcal{R}_k,\; m_{k-1}^{\ell-\hat{s}_k}(\ell-\hat{s}_k+1)\,\hat{B}_{k,\hat{s}_k}\bigr\} & \text{if } \ell \geq \hat{s}_k, \end{cases}
\end{equation}
with boundary value $\hat{B}_{k,\hat{s}_k} := C_k\sum_{k'<k}\hat{G}_{k',\hat{s}_k-1}$, and
\begin{equation}\label{eq:surrogate_G_piecewise}
\hat{G}_{k,\ell} = \begin{cases} \hat{G}_{k,s_k-1} & \text{if } \ell < s_k, \\ m_k^{\ell-s_k+1}(\ell-s_k+2)\,\hat{G}_{k,s_k-1} & \text{if } \ell \geq s_k, \end{cases}
\end{equation}
with initialization $\hat{G}_{k,s_k-1} = D_{k,s_k-1} + \hat{B}_{k,s_k-1} + \hat{B}_{k,s_k-2}$.
Note $\hat{G}_{k,s_k-1} \leq 3\mathcal{R}_k$ by \hl{\cref{lem:bounded_recovery}} and the $\hat{B}_{k,\cdot} \leq \mathcal{R}_k$ ceiling in~\eqref{eq:surrogate_piecewise}.

By \cref{lem:surrogate}, applying the surrogate domination directly,
\[
G_{k,\ell} \leq \hat{G}_{k,\ell} \leq m_k^{\ell-s_k+1}(\ell-s_k+2)\,\hat{G}_{k,s_k-1} \leq 3\mathcal{R}_k(\ell-s_k+2)\,m_k^{\ell-s_k+1}
\quad \text{for all } \ell \geq s_k.
\]
It remains to exhibit the conditions on $s_{k+1}$ under which \cref{lem:surrogate} applies.

\paragraph{Conditions arising in the proof of \cref{lem:surrogate}.}
The proof of \cref{lem:surrogate} (given below) requires five inequalities; we collect them here so that the present theorem can secure them via the choice of $s_{k+1}$.
For each condition we identify the proof step at which it is invoked.
\begin{enumerate}[label=\textbf{(C\arabic*)},leftmargin=*,nosep]
\item\label{cond:fold} \textbf{Predecessor envelope folding (per-predecessor form).} For every $k' \leq k$,
\begin{equation*}
m_{k'}^{\hat{s}_{k+1} - s_{k'}}(\hat{s}_{k+1} - s_{k'} + 1) \;\leq\; \tfrac{1}{k}.
\end{equation*}
\emph{Origin:} used in the bound on $B_{k+1,\ell}$ in \cref{lem:surrogate} (case $\ell \geq \hat{s}_{k+1}$) to fold the polynomial factor $(\hat{s}_{k+1} - s_{k'} + 1)$ inherited from each predecessor surrogate $\hat{G}_{k',\cdot}$ into the target geometric envelope and then average the resulting sum of $k$ terms back to a single $\hat{B}_{k+1,\hat{s}_{k+1}}$ value.
The per-predecessor form (with $m_{k'}$ rather than the global $m_k$) is no looser than the corresponding global statement, since $m_k \geq m_{k'}$ for $k' \leq k$ (proved below).

\item\label{cond:gap} \textbf{Wedin-gap activation}: $\norm{\bfY_{k+1,\ell}-\bfY_{k+1}^\star}_2 < \tfrac{1}{2}\calT_{k+1}^\star$ for all $\ell \geq \hat{s}_{k+1}$, equivalently $\hat{s}_{k+1} \geq s_{k+1}^{\mathrm{gap}}$ where $s_{k+1}^{\mathrm{gap}}$ is the smallest round such that $\sum_{k'\leq k} G_{k',\ell-1} \leq \calT_{k+1}^\star/(2C_{k+1})$.
\emph{Origin:} the hypothesis of \cref{lem:B_bound}; without it, $B_{k+1,\ell}$ has no usable upper bound.

\item\label{cond:gamma} \textbf{Coupling compatibility}: $\gamma_k := \tfrac{1}{k+1} + \tfrac{k}{k+1}m_k \leq m_{k+1}$.
\emph{Origin:} structural; required when invoking \cref{lem:Bhat_geometric} to re-anchor $\hat B_{k+1,\cdot}$ to a pure-geometric envelope at rate $\gamma_k$ and then absorb that rate into $m_{k+1}$.
This holds automatically because $m_{k+1} = \max\{\calF_{k+1}, \gamma_k\}$ by~\eqref{eq:mk_recurrence}.

\item\label{cond:contraction} \textbf{Contraction compatibility}: $\calF_{k+1} \leq m_{k+1}$.
\emph{Origin:} used to absorb the factor $\calF_{k+1}$ from \cref{lem:D_recurrence} into the target rate $m_{k+1}$ when bounding $D_{k+1,\ell}$.
This holds automatically because $m_{k+1} = \max\{\calF_{k+1}, \gamma_k\}$ by~\eqref{eq:mk_recurrence}.

\item\label{cond:reanchor} \textbf{Re-anchor warm-up offset}: the offset between the Wedin-activation index $\hat{s}_{k+1}$ and the warm-up index $s_{k+1}$ satisfies both
\begin{equation}\label{eq:reanchor_offset}
s_{k+1} - \hat{s}_{k+1} - 2 \;\geq\; \frac{(k+1) m_k}{1 - m_k} \qquad\text{and}\qquad m_k^{s_{k+1}-\hat s_{k+1}-2}(s_{k+1}-\hat s_{k+1}-1) \;\leq\; 1.
\end{equation}
\emph{Origin:} the two warm-up requirements of \cref{lem:Bhat_geometric} (with the second instance $s = s_{k+1}-2$ giving the tightest case): together they convert the linear-times-geometric form $m_k^{\ell-\hat s_{k+1}}(\ell-\hat s_{k+1}+1)\hat B_{k+1,\hat s_{k+1}}$ into the pure-geometric form $\gamma_k^{\ell-s_{k+1}+1}\hat B_{k+1,s_{k+1}-1}$ used to majorize the forcing sum in the unrolled $D$-recurrence at rate $m_{k+1}$ without incurring an additional polynomial factor.
\end{enumerate}
Conditions \ref{cond:gamma} and \ref{cond:contraction} are automatic structural facts (no constraint on $s_{k+1}$); \ref{cond:fold}, \ref{cond:gap}, and \ref{cond:reanchor} translate, via \cref{lem:lambert}(ii) (for the first two) or directly~\eqref{eq:reanchor_offset} (for the third), into explicit per-predecessor lower bounds on $s_{k+1}$.
Before proceeding, record once that $m_k$ is non-decreasing in $k$: by~\eqref{eq:mk_recurrence}, $\tfrac{1}{k} + \tfrac{k-1}{k}m_{k-1}$ is a convex combination of $1$ and $m_{k-1}$, hence $\geq m_{k-1}$, so $m_k = \max\{\calF_k,\,\gamma_{k-1}\} \geq m_{k-1}$.

\paragraph{Deriving~\eqref{eq:sk_condition} from \ref{cond:fold} and \ref{cond:gap}.}
We compute the per-predecessor lower bound on $s_{k+1}$ that simultaneously secures the two non-trivial conditions.

\emph{From \ref{cond:fold}.}
For each $k' \leq k$, we seek the smallest $\hat{s}_{k+1}$ such that $m_{k'}^{\hat{s}_{k+1}-s_{k'}}(\hat{s}_{k+1}-s_{k'}+1) \leq 1/k$.
Applying \cref{lem:lambert}(ii) to $g(x) = m_{k'}^x(x+1)$ with $x = \hat{s}_{k+1} - s_{k'}$ and $\epsilon = 1/k$ yields $\hat s_{k+1} - s_{k'} \geq \tfrac{1}{\log m_{k'}}W_{-1}(\tfrac{m_{k'}\log m_{k'}}{k}) - 1$.
Using the wrapper $\What(a) := \max\{1, -W_{-1}(-a)\}$, the trailing $-1$ is absorbed (since $\What(a) \geq 1$ for $a \in (0, 1/e)$ and the wrapper bounds $-W_{-1}(-a) - 1 \leq \What(a) - 1 + 1 = \What(a)$ once we pass to the inverse-sign convention $\epsilon \mapsto -\epsilon m \log m$):
\begin{equation}\label{eq:fold_bound}
\hat{s}_{k+1} - s_{k'} \;\geq\; \frac{\What\!\left(\tfrac{m_{k'}|\log m_{k'}|}{k}\right)}{\log(1/m_{k'})}, \qquad \forall k' \leq k.
\end{equation}

\emph{From \ref{cond:gap}.}
By \cref{lem:Y_mismatch} and the inductive hypothesis $G_{k',\ell-1} \leq 3\mathcal{R}_{k'}(\ell-s_{k'}+1)m_{k'}^{\ell-s_{k'}}$,
\begin{equation}
\norm{\bfY_{k+1,\ell}-\bfY_{k+1}^\star}_F \leq \sum_{k'\leq k} G_{k',\ell-1} \leq 3\sum_{k'\leq k}\mathcal{R}_{k'}(\ell-s_{k'}+1)m_{k'}^{\ell-s_{k'}}.
\end{equation}
A sufficient condition for the Wedin gap is that each summand is at most $\calT_{k+1}^\star/(6kC_{k+1})$, i.e., for each $k' \leq k$,
$(\ell - s_{k'}+1)m_{k'}^{\ell-s_{k'}} \leq \tfrac{\calT_{k+1}^\star}{6k\mathcal{R}_{k'}C_{k+1}}$ (the constant $6 = 3\times 2$ comes from $\hat G_{k',\cdot} \leq 3\mathcal{R}_{k'}\cdot m_{k'}^{\cdot}(\cdot+1)$ post-warm-up and the factor-of-two slack in \cref{lem:B_bound}'s gap hypothesis).
Applying \cref{lem:lambert}(ii) with $\epsilon_{k'} := \tfrac{\calT_{k+1}^\star}{6k\mathcal{R}_{k'}C_{k+1}}$ and the same $\What$-absorption,
\begin{equation}\label{eq:gap_bound}
\ell - s_{k'} \;\geq\; \frac{\What\!\left(\epsilon_{k'}\, m_{k'}|\log m_{k'}|\right)}{\log(1/m_{k'})}, \qquad \forall k' \leq k.
\end{equation}

\emph{From \ref{cond:reanchor}.}
The first inequality in~\eqref{eq:reanchor_offset} is an offset between $s_{k+1}$ and $\hat s_{k+1}$ that does not depend on $k'$; it contributes the additive term $(k+1)m_k/(1-m_k) + 2$.
The second inequality $m_k^{s_{k+1}-\hat s_{k+1}-2}(s_{k+1}-\hat s_{k+1}-1) \le 1$ is a Lambert-$W$ condition of the same form as~\eqref{eq:fold_bound} (applied to $g(x) = m_k^x(x+1)$ with $\epsilon = 1$), so it contributes an additional offset $\What(m_k|\log m_k|)/\log(1/m_k)$ between $\hat s_{k+1}$ and $s_{k+1}-2$.

\emph{Combining.}
The bounds~\eqref{eq:fold_bound} and~\eqref{eq:gap_bound} are per-predecessor lower bounds on the gap from anchor $s_{k'}$, while~\eqref{eq:reanchor_offset} adds a uniform offset between $\hat s_{k+1}$ and $s_{k+1}$, so a sufficient choice of $s_{k+1}$ is
\begin{equation}\label{eq:sk_condition}
\begin{aligned}
s_{k+1} \;\geq\;& \max_{k'\leq k}\!\Biggl\{ s_{k'} \;+\; \underbrace{\frac{\What\!\left(\tfrac{m_{k'}|\log m_{k'}|}{k}\right)}{\log(1/m_{k'})}}_{\text{from \ref{cond:fold},~\eqref{eq:fold_bound}}} \;+\; \underbrace{\frac{\What\!\left(\tfrac{\calT_{k+1}^\star\,m_{k'}|\log m_{k'}|}{6k\mathcal{R}_{k'}C_{k+1}}\right)}{\log(1/m_{k'})}}_{\text{from \ref{cond:gap},~\eqref{eq:gap_bound}}} \Biggr\}\\
&\;+\; \underbrace{\frac{(k+1) m_k}{1 - m_k} + 2 \;+\; \frac{\What\!\bigl(m_k|\log m_k|\bigr)}{\log(1/m_k)}}_{\text{from \ref{cond:reanchor},~\eqref{eq:reanchor_offset}}}.
\end{aligned}
\end{equation}
By \cref{lem:lambert} the wrapper $\What(\cdot)$ is well-defined and $\What(a) = \calO(\max\{1,\log(1/a)\})$, and the re-anchor offset is bounded for any fixed $m_k < 1$, so~\eqref{eq:sk_condition} produces a finite, polynomial-in-$\log$-and-$1/(1-m_k)$ requirement.

\paragraph{Simplified bound on $s_{k+1}$.}
The three nested $\What$ terms in~\eqref{eq:sk_condition} can be folded into a single logarithmic expression by applying the standard Lambert-$W$ upper bound (sharpening the asymptotic of \cref{lem:lambert})
\begin{equation}\label{eq:lambert_log_bound}
\What(a) \;\leq\; \log(1/a) + \log\log(1/a) + 1 \qquad \text{for } a \in (0, 1/e),
\end{equation}
together with the elementary bound $\log(1/m) \geq 1 - m$ for $m \in (0,1)$ (so that $1/\log(1/m_{k'}) \leq 1/(1-m_{k'})$).
For each predecessor $k' \leq k$, applying~\eqref{eq:lambert_log_bound} to $a = m_{k'}|\log m_{k'}|/k$ gives
\(
\What\!\bigl(\tfrac{m_{k'}|\log m_{k'}|}{k}\bigr) \leq \log\!\bigl(\tfrac{k}{m_{k'}|\log m_{k'}|}\bigr) + \log\log\!\bigl(\tfrac{k}{m_{k'}|\log m_{k'}|}\bigr) + 1,
\)
and the same expansion applied to the \ref{cond:gap}-term $\What$ replaces $k$ by $6kC_{k+1}\mathcal{R}_{k'}/\calT_{k+1}^\star$.
Absorbing the $\log\log$ and $+1$ remainders into a generic $\widetilde\calO(\cdot)$ symbol, \eqref{eq:sk_condition} simplifies to:

\begin{corollary}[Simplified warm-up bound]\label{cor:sk_simplified}
There exist absolute constants $C_1, C_2, C_3 > 0$ such that the per-predecessor lower bound~\eqref{eq:sk_condition} on $s_{k+1}$ is implied by
\begin{equation}\label{eq:sk_simplified}
s_{k+1} \;\geq\; \max_{k'\leq k}\!\left\{ s_{k'} + \frac{C_1\,\log\!\bigl(k\,C_{k+1}\mathcal{R}_{k'}/\calT_{k+1}^\star\bigr) + C_2}{1 - m_{k'}} \right\} \;+\; \frac{(k+1) m_k}{1 - m_k} \;+\; \frac{C_3}{1 - m_k}.
\end{equation}
Equivalently, in $\widetilde\calO$-form (suppressing constants and $\log\log$ factors),
\begin{equation}\label{eq:sk_simplified_otilde}
s_{k+1} \;=\; \max_{k'\leq k}\, s_{k'} \;+\; \widetilde\calO\!\left(\frac{\log\!\bigl(k\,C_{k+1}\mathcal{R}_{k'}/\calT_{k+1}^\star\bigr)}{1 - m_k}\right) \;+\; \frac{(k+1) m_k}{1 - m_k}.
\end{equation}
\end{corollary}
\begin{proof}
Each $\What$ term in~\eqref{eq:sk_condition} is bounded via~\eqref{eq:lambert_log_bound}. The denominator $\log(1/m_{k'})$ is replaced by its lower bound $1 - m_{k'}$, which is uniformly bounded above by $1 - m_k$ since $m_{k'} \leq m_k$ (\cref{eq:mk_recurrence}). The two $\What$ contributions inside the $\max$ each yield a $\log(k/(m_{k'}|\log m_{k'}|))$ piece and a $\log(C_{k+1}\mathcal{R}_{k'}/\calT_{k+1}^\star)$ piece (the latter only from the \ref{cond:gap}-term); both are absorbed into $C_1\log(k\,C_{k+1}\mathcal{R}_{k'}/\calT_{k+1}^\star)$. The remaining $\log\log$ and $+1$ terms are bounded by a constant $C_2$, while the third (re-anchor) $\What$ contributes a constant $C_3$ once $1/\log(1/m_k) \leq 1/(1-m_k)$ is applied. Combining yields~\eqref{eq:sk_simplified}; \eqref{eq:sk_simplified_otilde} is the same statement with constants and $\log\log$ factors hidden.
\end{proof}
The simplified form~\eqref{eq:sk_simplified} is the version referenced in the main-text statement of \cref{thm:parallel_convergence}; it makes explicit the three drivers of the warm-up cost: the predecessor-anchor offset $s_{k'}$, a logarithmic dependence on the spectral-gap-to-magnitude ratio $C_{k+1}\mathcal{R}_{k'}/\calT_{k+1}^\star$, and the geometric $1/(1-m_k)$ slowdown as the effective rate approaches one.

\paragraph{Conclusion.}
With $s_{k+1}$ chosen per~\eqref{eq:sk_condition}, conditions \ref{cond:fold}, \ref{cond:gap}, and \ref{cond:reanchor} are secured; \ref{cond:gamma} and \ref{cond:contraction} are automatic from~\eqref{eq:mk_recurrence}.
\cref{lem:surrogate} then yields $G_{k,\ell} \leq \hat{G}_{k,\ell}$ for all $\ell \geq s_k - 1$, and substituting~\eqref{eq:surrogate_G_piecewise} together with $\hat{G}_{k,s_k-1} \leq 3\mathcal{R}_k$ gives
\[ G_{k,\ell} \leq m_k^{\ell-s_k+1}(\ell-s_k+2)\,\hat{G}_{k,s_k-1} \leq 3\mathcal{R}_k(\ell - s_k + 2)\,m_k^{\ell-s_k+1}, \]
which is the bound~\eqref{eq:main_bound}.
This the proof.
\end{proof}

\subsection{Auxiliary Lemmas}

\begin{lemma}[Lambert-$W$ Bound]\label{lem:lambert}
Let $m \in (0,1)$ and $g(x) = m^x(x+1)$.
\begin{enumerate}[nosep,leftmargin=*]
\item[(i)] If $\epsilon \geq -(em\log m)^{-1}$, then $g(x) \leq \epsilon$ for all $x \geq 0$.
\item[(ii)] If $\epsilon < -(em\log m)^{-1}$, then $g(x) \leq \epsilon$ whenever $x \geq \tfrac{1}{\log m}W_{-1}(\epsilon m\log m) - 1$.
\end{enumerate}
Moreover, the wrapper $\What(a) := \max\{1, -W_{-1}(-a)\}$ satisfies $\What(a) = \calO(\max\{1,\log(1/a)\})$, and admits the explicit upper bound
\(
\What(a) \leq \log(1/a) + \log\log(1/a) + 1
\)
for $a \in (0, 1/e)$.
\end{lemma}
\begin{proof}
Differentiating $g(x) = m^x(x+1)$ gives $g'(x) = m^x\bigl(1 + (x+1)\log m\bigr)$, so $g$ attains its maximum at $x^\star = (\log(1/m))^{-1} - 1$ with value $g(x^\star) = -(em\log m)^{-1}$.

\emph{Case (i).}
If $\epsilon \geq g(x^\star)$, then $g(x) \leq g(x^\star) \leq \epsilon$ for every $x \geq 0$, and any non-negative $x$ satisfies the bound.

\emph{Case (ii).}
Suppose $\epsilon < g(x^\star)$.
On $(x^\star,\infty)$, $g$ is strictly decreasing, so the equation $g(x) = \epsilon$ has a unique solution $x_\epsilon > x^\star$.
Taking logarithms and substituting $y = (x+1)\log m$ transforms $g(x) = \epsilon$ into $y e^y = \epsilon m\log m$, which is solved by $y = W_{-1}(\epsilon m\log m)$ on the decreasing branch (since $\epsilon m\log m \in (-1/e,0)$).
Hence $x_\epsilon = \tfrac{1}{\log m}W_{-1}(\epsilon m\log m) - 1$, and every $x \geq x_\epsilon$ satisfies $g(x) \leq \epsilon$.

\emph{Asymptotic statement.}
By the standard expansion $-W_{-1}(-a) = \log(1/a) + \log\log(1/a) + o(1)$ as $a \to 0^+$, we get $\What(a) = \calO(\max\{1,\log(1/a)\})$.
The explicit non-asymptotic bound $\What(a) \leq \log(1/a) + \log\log(1/a) + 1$ for $a \in (0,1/e)$ follows from the standard Lambert-$W$ inequality $-W_{-1}(-e^{-u-1}) \leq u + \log u + 1$ for $u \geq 1$ (substitute $u = \log(1/a) - 1$).
\end{proof}

\begin{lemma}[Geometric Re-anchoring of $\hat B$]\label{lem:Bhat_geometric}
Fix $k \geq 1$.
Suppose the surrogate sequence $\hat B_{k+1,\ell}$ is given by~\eqref{eq:surrogate_B} with $\hat B_{k+1,\hat s_{k+1}} \leq \mathcal{R}_{k+1}$ and rate $m_k \in (0,1)$.
If an integer $s$ satisfies $s \geq \hat s_{k+1}$, $m_k^{s-\hat s_{k+1}}(s-\hat s_{k+1}+1) \leq 1$, and
\begin{equation}\label{eq:reanchor_warmup}
s - \hat s_{k+1} \;\geq\; \frac{(k+1) m_k}{1 - m_k},
\end{equation}
then for every $\ell \geq s$,
\begin{equation}\label{eq:Bhat_geometric}
\hat B_{k+1,\ell} \;\leq\; \gamma_k^{\,\ell - s}\, \hat B_{k+1,s}, \qquad \gamma_k = \tfrac{1}{k+1} + \tfrac{k}{k+1}m_k.
\end{equation}
\end{lemma}
\begin{proof}
At index $s$, the conditions $\hat B_{k+1,\hat s_{k+1}} \leq \mathcal{R}_{k+1}$ and $m_k^{s-\hat s_{k+1}}(s-\hat s_{k+1}+1) \leq 1$ together imply $m_k^{s-\hat s_{k+1}}(s-\hat s_{k+1}+1)\hat B_{k+1,\hat s_{k+1}} \leq \mathcal{R}_{k+1}$, so the $\min\{\cdot,\mathcal{R}_{k+1}\}$ ceiling is inactive at $s$ and
\[
\hat B_{k+1,s} = m_k^{s-\hat s_{k+1}}(s-\hat s_{k+1}+1)\hat B_{k+1,\hat s_{k+1}}, \quad\text{so}\quad \hat B_{k+1,\hat s_{k+1}} = \bigl[m_k^{s-\hat s_{k+1}}(s-\hat s_{k+1}+1)\bigr]^{-1}\hat B_{k+1,s}.
\]
For any $\ell \geq s$, the post-warm-up clause of~\eqref{eq:surrogate_B} gives
\begin{align*}
\hat B_{k+1,\ell}
&\leq m_k^{\ell-\hat s_{k+1}}(\ell-\hat s_{k+1}+1)\hat B_{k+1,\hat s_{k+1}} \\
&= m_k^{\ell-s}\cdot\frac{\ell-\hat s_{k+1}+1}{s-\hat s_{k+1}+1}\,\hat B_{k+1,s} \\
&= m_k^{\ell-s}\,\hat B_{k+1,s}\prod_{\ell'=s+1}^{\ell}\frac{\ell'-\hat s_{k+1}+1}{\ell'-\hat s_{k+1}}
\;\leq\; m_k^{\ell-s}\!\left(\frac{s-\hat s_{k+1}+1}{s-\hat s_{k+1}}\right)^{\ell-s}\hat B_{k+1,s},
\end{align*}
where the last inequality bounds each of the $\ell-s$ factors by its largest value at $\ell' = s+1$ (the ratio $(\ell'-\hat s_{k+1}+1)/(\ell'-\hat s_{k+1})$ is decreasing in $\ell'$).
The warm-up condition~\eqref{eq:reanchor_warmup}, equivalent to $\frac{s-\hat s_{k+1}+1}{s-\hat s_{k+1}} \leq 1 + \frac{1}{(k+1)m_k/(1-m_k)} = \frac{(k+1)m_k + 1 - m_k}{(k+1)m_k} = \frac{km_k + 1}{(k+1)m_k}$, then yields
\[
m_k\cdot\frac{s-\hat s_{k+1}+1}{s-\hat s_{k+1}} \;\leq\; \frac{km_k + 1}{k+1} \;=\; \tfrac{1}{k+1} + \tfrac{k}{k+1}m_k \;=\; \gamma_k,
\]
which combined with the displayed bound gives~\eqref{eq:Bhat_geometric}.
\end{proof}

\begin{lemma}[Surrogate Sequence Domination]\label{lem:surrogate}
Under \cref{asump:contraction,asump:gap,asump:bounded_recovery}, let $\hat{s}_k$ denote the earliest round at which the Wedin gap condition of \cref{lem:B_bound} holds for component $k$, and let $s_k$ be the warm-up round chosen via~\eqref{eq:sk_condition}.
The choice~\eqref{eq:sk_condition} ensures $\hat{s}_k \leq s_k$ for every $k \geq 2$ via the additive offset $(k+1)m_k/(1-m_k) + 2 + \What(m_k|\log m_k|)/\log(1/m_k)$ contributed by~\ref{cond:reanchor}; for indices $j < \hat{s}_k$ we adopt the convention $\hat{B}_{k,j} := \mathcal{R}_k$.
Define the surrogate sequences piecewise:
\begin{align}
\hat{B}_{k,\ell} &= \begin{cases} \mathcal{R}_k & \text{if } \ell < \hat{s}_k, \\ \min\bigl\{\mathcal{R}_k,\; m_{k-1}^{\ell-\hat{s}_k}(\ell-\hat{s}_k+1)\,\hat{B}_{k,\hat{s}_k}\bigr\} & \text{if } \ell \geq \hat{s}_k, \end{cases} \label{eq:surrogate_B}\\
\hat{G}_{k,\ell} &= \begin{cases} \hat{G}_{k,s_k-1} & \text{if } \ell < s_k, \\ m_k^{\ell-s_k+1}(\ell-s_k+2)\,\hat{G}_{k,s_k-1} & \text{if } \ell \geq s_k, \end{cases} \label{eq:surrogate_G}
\end{align}
with boundary values $\hat{B}_{k,\hat{s}_k} := C_k\sum_{k'<k}\hat{G}_{k',\hat{s}_k-1}$ and $\hat{G}_{k,s_k-1} := D_{k,s_k-1} + \hat{B}_{k,s_k-1} + \hat{B}_{k,s_k-2}$.
If $s_{k+1}$ is chosen via~\eqref{eq:sk_condition}, then $\hat{B}_{k,\ell} \geq B_{k,\ell}$ for every $k \in [r]$ and every $\ell \geq 0$, and $\hat{G}_{k,\ell} \geq G_{k,\ell}$ for every $\ell \geq s_k - 1$.
\end{lemma}
\begin{proof}
We proceed by strong induction on $k$.
At each inequality below we identify which of the conditions \ref{cond:fold}--\ref{cond:contraction} (collected in the proof of \cref{thm:parallel_convergence}) is being invoked.

\emph{Base case ($k=1$).}
Since $\bfY_{1,\ell} = \bfY$, the first target is exact, so $B_{1,\ell}=0$ by~\eqref{eq:def_B}; combined with $B_{1,\ell} \leq \mathcal{R}_1$ from \hl{\cref{lem:bounded_recovery}}, the surrogate $\hat{B}_{1,\ell} = \mathcal{R}_1$ dominates trivially.
For the total error, \cref{lem:D_recurrence} unrolls to $G_{1,\ell} = D_{1,\ell} \leq \calF_1^{\ell-s_1+1}D_{1,s_1-1}$, which is upper-bounded by $m_1^{\ell-s_1+1}(\ell-s_1+2)\hat{G}_{1,s_1-1} = \hat{G}_{1,\ell}$ since $m_1 = \calF_1$ and $(\ell-s_1+2) \geq 1$.

\emph{Inductive step.}
Fix $k \geq 1$ and assume the inductive hypothesis: for every $k' \leq k$, $\hat{B}_{k',\ell} \geq B_{k',\ell}$ for every $\ell \geq 0$ and $\hat{G}_{k',\ell} \geq G_{k',\ell}$ for every $\ell \geq s_{k'}-1$.
We show $\hat{B}_{k+1,\ell} \geq B_{k+1,\ell}$ for every $\ell \geq 0$, and then $\hat{G}_{k+1,\ell} \geq G_{k+1,\ell}$ for every $\ell \geq s_{k+1}-1$.

We bound $B_{k+1,\ell}$ first.
If $\ell < \hat{s}_{k+1}$, then by \hl{\cref{lem:bounded_recovery}}, $B_{k+1,\ell} \leq \mathcal{R}_{k+1} = \hat{B}_{k+1,\ell}$.
If $\ell \geq \hat{s}_{k+1}$, then the Wedin-gap hypothesis of \cref{lem:B_bound} is active by condition \ref{cond:gap}, and by definition of $\hat{s}_{k+1}$, $\hat{s}_{k+1}-1 \geq s_{k'}$ for every $k' \leq k$, so each predecessor surrogate lies on the post-warm-up branch of~\eqref{eq:surrogate_G}.
Combining \cref{lem:Y_mismatch} and \cref{lem:B_bound} with the inductive hypothesis $G_{k',\ell-1} \leq \hat{G}_{k',\ell-1}$,
\begin{align*}
B_{k+1,\ell}
&\leq C_{k+1}\sum_{k'\leq k} \hat{G}_{k',\ell-1} \\
&= C_{k+1}\sum_{k'\leq k} m_{k'}^{\ell-s_{k'}}(\ell-s_{k'}+1)\,\hat{G}_{k',s_{k'}-1} \\
&\leq C_{k+1}\sum_{k'\leq k} m_{k'}^{\ell-\hat{s}_{k+1}}\,m_{k'}^{\hat{s}_{k+1}-s_{k'}}(\ell-\hat{s}_{k+1}+1)(\hat{s}_{k+1}-s_{k'}+1)\,\hat{G}_{k',s_{k'}-1} \\
&\leq C_{k+1}\,m_k^{\ell-\hat{s}_{k+1}}(\ell-\hat{s}_{k+1}+1)\cdot\tfrac{1}{k}\sum_{k'\leq k}\hat{G}_{k',s_{k'}-1},
\end{align*}
where the third inequality applies $a+b-1 \leq ab$ for $a, b \geq 1$ (with $a = \ell-\hat{s}_{k+1}+1$ and $b = \hat{s}_{k+1}-s_{k'}+1$) to split the polynomial factor along $\hat{s}_{k+1}$, and the fourth uses condition~\ref{cond:fold} ($m_{k'}^{\hat{s}_{k+1}-s_{k'}}(\hat{s}_{k+1}-s_{k'}+1) \le 1/k$) together with the monotonicity $m_{k'} \leq m_k$ for $k' \leq k$ (from~\eqref{eq:mk_recurrence}).
By~\ref{cond:fold} again, $\hat{G}_{k',\hat{s}_{k+1}-1} = m_{k'}^{\hat{s}_{k+1}-s_{k'}}(\hat{s}_{k+1}-s_{k'}+1)\hat{G}_{k',s_{k'}-1} \leq (1/k)\hat{G}_{k',s_{k'}-1}$, so $\hat{B}_{k+1,\hat{s}_{k+1}} := C_{k+1}\sum_{k'\leq k}\hat{G}_{k',\hat{s}_{k+1}-1} \leq (C_{k+1}/k)\sum_{k'\leq k}\hat{G}_{k',s_{k'}-1}$, and combining with $B_{k+1,\ell} \leq \mathcal{R}_{k+1}$ from \hl{\cref{lem:bounded_recovery}},
\begin{equation}\label{eq:Bk+1_surrogate}
B_{k+1,\ell} \;\leq\; \min\bigl\{\mathcal{R}_{k+1},\; m_k^{\ell-\hat{s}_{k+1}}(\ell-\hat{s}_{k+1}+1)\,\hat{B}_{k+1,\hat{s}_{k+1}}\bigr\} \;=\; \hat{B}_{k+1,\ell},
\end{equation}
which is the post-warm-up branch of~\eqref{eq:surrogate_B}.
The $\min$-ceiling enforces $\hat B_{k+1,\hat s_{k+1}}\leq \mathcal{R}_{k+1}$, which we use when invoking \cref{lem:Bhat_geometric}.
This establishes $\hat{B}_{k+1,\ell} \geq B_{k+1,\ell}$ for every $\ell \geq 0$.

We now bound $G_{k+1,\ell}$, starting with $\ell \geq s_{k+1}$.
Iterating \cref{lem:D_recurrence} from round $s_{k+1}-1$ up to round $\ell$ and using~\eqref{eq:Bk+1_surrogate},
\begin{equation}\label{eq:D_unroll}
D_{k+1,\ell} \;\leq\; \calF_{k+1}^{\ell-s_{k+1}+1}D_{k+1,s_{k+1}-1} + \sum_{\ell'=s_{k+1}-1}^{\ell}\calF_{k+1}^{\ell-\ell'}\bigl(\hat B_{k+1,\ell'} + \hat B_{k+1,\ell'-1}\bigr),
\end{equation}
where we re-indexed the sum to start at $s_{k+1}-1$ (the dropped earlier summand $B_{k+1,s_{k+1}-2}$ is non-negative).
To collapse this sum into a clean envelope at rate $m_{k+1}$, we re-anchor $\hat B_{k+1,\cdot}$ from $\hat s_{k+1}$ to $s_{k+1}-1$ via \cref{lem:Bhat_geometric}.
Apply that lemma twice, at $s = s_{k+1}-1$ and at $s = s_{k+1}-2$:
\begin{align}
\hat B_{k+1,\ell'} &\;\leq\; \gamma_k^{\,\ell' - (s_{k+1}-1)}\,\hat B_{k+1,s_{k+1}-1}, \qquad \forall\,\ell' \geq s_{k+1}-1, \label{eq:reanchor_apply}\\
\hat B_{k+1,\ell'-1} &\;\leq\; \gamma_k^{\,\ell' - (s_{k+1}-1)}\,\hat B_{k+1,s_{k+1}-2}, \qquad \forall\,\ell' \geq s_{k+1}-1, \label{eq:reanchor_apply_shifted}
\end{align}
where the second display has matching exponent $(\ell'-1)-(s_{k+1}-2) = \ell'-s_{k+1}+1$.
The two preconditions of \cref{lem:Bhat_geometric} at $s = s_{k+1}-2$, namely $m_k^{s_{k+1}-\hat s_{k+1}-2}(s_{k+1}-\hat s_{k+1}-1) \leq 1$ and $s_{k+1}-\hat s_{k+1}-2 \geq (k+1)m_k/(1-m_k)$, are exactly the two clauses of \ref{cond:reanchor} secured by~\eqref{eq:sk_condition}; at $s = s_{k+1}-1$ the offset is one larger and so no looser, and $g(x) = m_k^x(x+1)$ is decreasing past its peak at $x = 1/\log(1/m_k) - 1$, which lies below \ref{cond:reanchor}'s offset (using $1/\log(1/m_k) \leq 1/(1-m_k)$).

Substituting~\eqref{eq:reanchor_apply} and~\eqref{eq:reanchor_apply_shifted} into~\eqref{eq:D_unroll},
\begin{align*}
D_{k+1,\ell}
&\leq m_{k+1}^{\ell-s_{k+1}+1}D_{k+1,s_{k+1}-1} \;+\; \sum_{\ell'=s_{k+1}-1}^{\ell} m_{k+1}^{\ell-\ell'}\,\gamma_k^{\,\ell'-s_{k+1}+1}\bigl(\hat B_{k+1,s_{k+1}-1}+\hat B_{k+1,s_{k+1}-2}\bigr)\\
&\leq m_{k+1}^{\ell-s_{k+1}+1}D_{k+1,s_{k+1}-1} \;+\; m_{k+1}^{\ell-s_{k+1}+1}(\ell-s_{k+1}+2)\bigl(\hat B_{k+1,s_{k+1}-1}+\hat B_{k+1,s_{k+1}-2}\bigr) \\
&\leq m_{k+1}^{\ell-s_{k+1}+1}(\ell-s_{k+1}+2)\bigl(D_{k+1,s_{k+1}-1} + \hat B_{k+1,s_{k+1}-1} + \hat B_{k+1,s_{k+1}-2}\bigr),
\end{align*}
where the first inequality uses \ref{cond:contraction} ($\calF_{k+1}\leq m_{k+1}$), and the second bounds the inner sum by $(\ell-s_{k+1}+2)$ copies of $m_{k+1}^{\ell-s_{k+1}+1}$ via \ref{cond:gamma} ($\gamma_k \leq m_{k+1}$) over $\ell-s_{k+1}+2$ consecutive indices.
By \eqref{eq:Bk+1_surrogate}, \eqref{eq:reanchor_apply}, and $\gamma_k \leq m_{k+1}$, we also have $B_{k+1,\ell} \leq \hat B_{k+1,\ell} \leq m_{k+1}^{\ell-s_{k+1}+1}(\ell-s_{k+1}+2)\,\hat B_{k+1,s_{k+1}-1}$.
Combining via $G_{k+1,\ell} \leq D_{k+1,\ell} + B_{k+1,\ell}$,
\begin{align*}
G_{k+1,\ell}
&\leq m_{k+1}^{\ell-s_{k+1}+1}(\ell-s_{k+1}+2)\bigl(D_{k+1,s_{k+1}-1} + \hat{B}_{k+1,s_{k+1}-1} + \hat{B}_{k+1,s_{k+1}-2}\bigr) \\
&= m_{k+1}^{\ell-s_{k+1}+1}(\ell-s_{k+1}+2)\,\hat{G}_{k+1,s_{k+1}-1} \;=\; \hat{G}_{k+1,\ell}.
\end{align*}

For the boundary case $\ell = s_{k+1}-1$, the surrogate is held constant: $\hat{G}_{k+1,s_{k+1}-1} = D_{k+1,s_{k+1}-1} + \hat{B}_{k+1,s_{k+1}-1} + \hat{B}_{k+1,s_{k+1}-2}$, so by the just-established $\hat B_{k+1,j} \geq B_{k+1,j}$ and $\hat B_{k+1,s_{k+1}-2} \geq 0$,
\[
\hat{G}_{k+1,s_{k+1}-1} \;\geq\; D_{k+1,s_{k+1}-1} + B_{k+1,s_{k+1}-1} \;\geq\; G_{k+1,s_{k+1}-1},
\]
where the last step is $G_{k+1,\cdot} \leq D_{k+1,\cdot} + B_{k+1,\cdot}$.
The size bound $\hat G_{k+1,s_{k+1}-1} \leq 3\mathcal{R}_{k+1}$ used in~\eqref{eq:main_bound} follows from \hl{\cref{lem:bounded_recovery}} applied to each of the three summands ($D_{k+1,s_{k+1}-1} \leq \mathcal{R}_{k+1}$ and $\hat B_{k+1,j} \leq \mathcal{R}_{k+1}$ by the $\min$-ceiling in~\eqref{eq:surrogate_B}).
This completes the inductive step.
\end{proof}

\subsection{Proof of Proposition~\ref{prop:self_correction}}\label{app:proof_self_correction}

\begin{proof}
By \cref{lem:Y_mismatch}: $\norm{\bfY_{k,\ell} - \bfY_k^\star}_F \leq \sum_{k'<k} G_{k',\ell-1}$.
Substituting \cref{thm:parallel_convergence}: each term is $3\mathcal{R}_{k'}(\ell - s_{k'} + 1)\,m_{k'}^{\ell-s_{k'}} \to 0$ since $m_{k'} < 1$.
\end{proof}

\subsection{Proof of Theorem~\ref{thm:gen_noiseless_parallel}}\label{app:proof_gen_noiseless}

\begin{proof}
Since $\svmin(\bfX) > 0$, for any $\bfW$ we have $\|\bfW\|_F \leq \svmin(\bfX)^{-1}\|\bfW\bfX\|_F$.
Applying this with $\bfW = \bfW^\star - \sum_k \bfb_{k,L}\bfa_{k,L}^\top$ and using the triangle inequality:
$\norm{\bfW^\star\bfX - \sum_k \bfb_{k,L}\bfa_{k,L}^\top\bfX}_F
\leq \norm{\bfY - \sum_{k=1}^r \bfb_k^\star\bfa_k^{\star\top}\bfX}_F + \sum_{k=1}^r \norm{\bfb_k^\star\bfa_k^{\star\top}\bfX - \bfb_{k,L}\bfa_{k,L}^\top\bfX}_F
= \sum_{k=r+1}^{r^\star}\sigma_k^\star + \sum_{k=1}^r G_{k,L}$,
where the first equality uses Eckart-Young-Mirsky: the best rank-$r$ approximation of $\bfY$ in Frobenius norm leaves residual $\sum_{k>r}\sigma_k^\star$.
Dividing by $\svmin(\bfX)$ yields the stated bound.
\end{proof}

\subsection{Proof of Theorem~\ref{thm:gen_noisy_parallel}}\label{app:proof_gen_noisy}

\begin{proof}
Decompose into $T_1$ (truncation, bounded by Eckart-Young-Mirsky), $T_2$ (noise perturbation, bounded by Weyl/Wedin and Gaussian concentration following~\citep[Lemma~C.1]{VandchaliOneRank2024}), and $T_3$ (algorithmic error $\sum_k G_{k,L}$).
The algorithmic term inherits self-correction; the statistical term matches the sequential baseline.
Dividing by $\svmin(\bfX)$ yields the stated bound.

The key observation is that parallelization does not affect the noise term $T_2$.
The noise $\calE$ enters through the data model $\bfY = \bfW^\star\bfX + \calE$, which is identical for both sequential and parallel algorithms.
The $T_2$ bound depends only on the spectral perturbation $\|\bfY - \bfW^\star\bfX\|_2 = \|\calE\|_2$, controlled by standard Gaussian concentration (e.g., $\|\calE\|_2 \leq \varepsilon(\sqrt{m} + \sqrt{n})$ with high probability).
This perturbation is a property of the data, not the algorithm, so the parallel and sequential statistical terms are the same.
The only difference is in $T_3$: for the parallel algorithm, $\sum_k G_{k,L} \to 0$ by \cref{thm:parallel_convergence}, while for the sequential baseline, the corresponding sum involves permanently fixed numerical errors.
\end{proof}

\subsection{Verification that $m_k < 1$}\label{app:mk_below_one}

We show by induction on $k$ that the effective rate $m_k$ defined in~\eqref{eq:mk_recurrence} satisfies $m_k \in (0,1)$ for every $k \in [r]$.

\emph{Base case ($k=1$).}
By definition, $m_1 = \calF_1$, and \cref{asump:contraction} requires $\calF_k \in (0,1)$ for every $k$.
Hence $m_1 \in (0,1)$.

\emph{Inductive step ($k \geq 2$).}
Assume $m_{k-1} \in (0,1)$.
The auxiliary quantity $\gamma_{k-1} = \tfrac{1}{k} + \tfrac{k-1}{k}m_{k-1}$ is a convex combination of $1$ and $m_{k-1}$, so $\gamma_{k-1} \in (m_{k-1},\,1) \subset (0,1)$ whenever $m_{k-1} < 1$.
By \cref{asump:contraction}, $\calF_k \in (0,1)$.
Therefore $m_k = \max\{\calF_k,\,\gamma_{k-1}\} \in (0,1)$, completing the induction.

\section{Game-Theoretic Interpretation}\label{app:proof_nash}

\begin{proposition}[Nash equilibrium]\label{prop:nash}
Under \cref{asump:gap}, define player $k$'s utility as
$\mathcal{V}_k(\bfa, \bfb \mid \{(\bfa_{k'}, \bfb_{k'})\}_{k'<k}) = -\frac{1}{2}\|\bfY - \sum_{k'<k}\bfb_{k'}\bfa_{k'}^\top\bfX - \bfb\bfa^\top\bfX\|_F^2$.
Then $\{(\bfa_k^\star, \bfb_k^\star)\}_{k=1}^r$ is the unique strict Nash equilibrium up to sign.
\end{proposition}
\begin{proof}
By induction: given optimal predecessors, player $k$ maximizes utility by finding the best rank-1 approximation of $\bfY_k^\star$, which is unique under the strict spectral gap (\cref{asump:gap}).
\end{proof}

\section{Theory-Validation Experiments}\label{app:theory_validation}

We evaluate the predictions of \cref{thm:parallel_convergence,prop:self_correction,thm:gen_noisy_parallel} on the bilinear regression model $\bfY = \bfW^\star \bfX + \calE$ from~\eqref{eq:main_problem}.
The synthetic instances use a ground-truth weight $\bfW^\star \in \R^{m \times d}$ of rank $r^\star$, an input matrix $\bfX \in \R^{d \times n}$ with i.i.d.\ standard Gaussian entries, and Gaussian noise $\calE_{ij} \sim \mathcal{N}(0, \varepsilon^2)$ entrywise, with $\varepsilon = 0$ in the noiseless experiments.
Three singular-value profiles for $\bfW^\star$ are considered, each normalised so that $\sigma_1^\star = 1$: \emph{exponential} $\sigma_k^\star = 2^{r^\star - k}$ (geometric decay, large gaps), \emph{power-law} $\sigma_k^\star = k^{-1.5}$ (algebraic decay), and \emph{uniform} $\sigma_k^\star = 1$ for all $k$ (no spectral gap, so \cref{asump:gap} is degenerate).
Each setting is repeated over five random seeds; figures plot the per-round median across seeds with shaded bands giving the seed-to-seed minimum-to-maximum envelope.

\subsection{Synthetic Validation}\label{app:exp_synthetic}

We first verify the qualitative predictions of \cref{thm:parallel_convergence,prop:self_correction,thm:gen_noisy_parallel}: the per-round convergence of total reconstruction error, the self-correction of deflation targets, and the matching of the parallel and sequential noise floors.

For the convergence comparison of \cref{fig:appendix-convergence} we set $m = 100$, $d = 200$, $n = 500$, $r^\star = 10$, and run each method with $T = 10$ inner iterations per round and $L = 10$ communication rounds across the three spectral profiles.
Parallel ALS matches sequential ALS on the exponential spectrum, with both reaching a relative weight error of about $2.1 \times 10^{-3}$.
On the power-law spectrum parallel reaches $3.6 \times 10^{-3}$ against $2.1 \times 10^{-3}$ for sequential, so the two stay within a small constant factor; on the uniform spectrum both methods degrade together because the spectral gap vanishes, consistent with \cref{asump:gap}.
For reference we also include two factored gradient-descent variants (parallel and sequential, using the rank-1 solver of \cref{app:subroutines} in place of ALS) and a joint gradient-descent baseline that updates all $r^\star$ rank-1 factor pairs simultaneously without deflation.

To isolate self-correction (\cref{fig:appendix-self-correction}) we shrink the problem to $m = 50$, $d = 80$, $n = 200$, $r^\star = 5$ on the exponential spectrum, with $T = 10$ and $L = 30$, so that workers $k = 2, \ldots, 5$ activate round by round under the staggered schedule $s_k = k$ of \cref{thm:parallel_convergence}.
The left panel shows the normalised deflation mismatch $\norm{\bfY_{k,\ell} - \bfY_k^\star}_F / \norm{\bfY_k^\star}_F$ for each worker, where $\bfY_{k,\ell}$ is the worker's current target and $\bfY_k^\star = \bfY - \sum_{k' < k} \bfb_{k'}^\star \bfa_{k'}^{\star\top} \bfX$ is the clean target built from the ideal rank-1 components $(\bfa_{k'}^\star, \bfb_{k'}^\star)$ of \cref{sec:error_decomp}.
The mismatch decreases by roughly $5\times$ for workers~2 to 4 between activation and convergence.
The right panel decomposes the error of worker~3 into the three quantities $G_{3,\ell}$, $D_{3,\ell}$, and $B_{3,\ell}$ defined in~\eqref{eq:def_D},~\eqref{eq:def_B},~and~\eqref{eq:def_G}.
Both $D_{3,\ell}$ and $B_{3,\ell}$ decay to a small common floor, exactly the behaviour predicted by \cref{prop:self_correction}.

The noise-robustness experiment of \cref{fig:appendix-noise} keeps the same problem dimensions and exponential profile, sets $T = 10$ and $L = 15$, and adds entrywise Gaussian noise $\calE_{ij} \sim \mathcal{N}(0, \varepsilon^2)$ at $\varepsilon \in \{0.01, 0.1, 0.5, 1.0\}$.
Across all four noise levels parallel and sequential deflation converge to the same noise floor, and the per-seed final-round parallel-to-sequential ratio stays within $[1.000, 1.014]$ across all four noise settings, so parallelization incurs no statistical penalty relative to sequential deflation, as predicted by \cref{thm:gen_noisy_parallel}.
The reference floor in each panel is the theoretical scale $\varepsilon \sqrt{r^\star d / n}$ from the same theorem.

\begin{figure}[t]
\centering
\includegraphics[width=\textwidth]{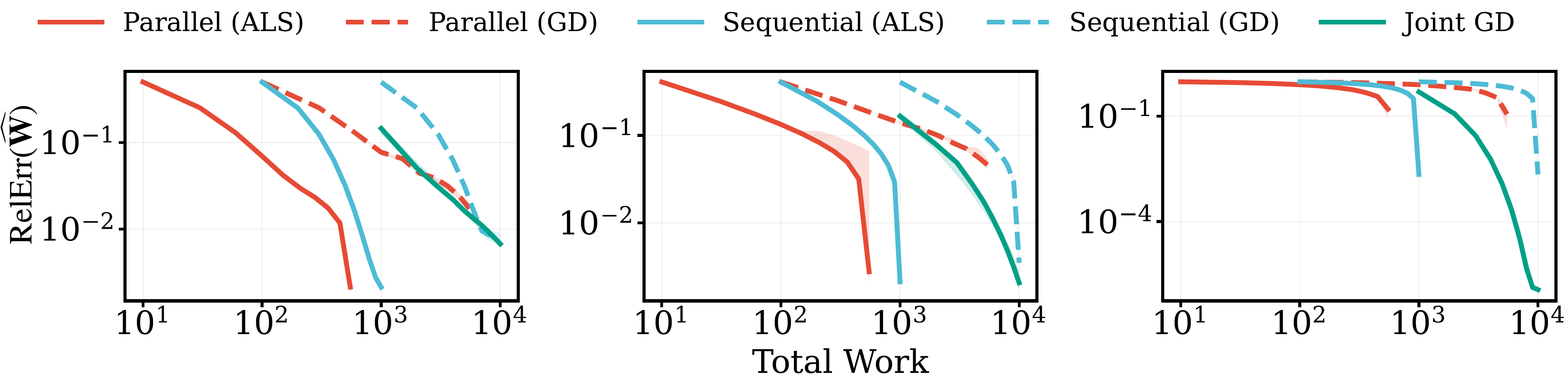}
\caption{\textbf{Convergence under equal total work.}
Relative weight error vs.\ total work for exponential, power-law, and uniform spectra.
Parallel ALS matches sequential on the well-separated exponential spectrum and degrades gracefully as spectral gaps shrink.}\label{fig:appendix-convergence}
\end{figure}

\begin{figure}[t]
\centering
\includegraphics[width=\textwidth]{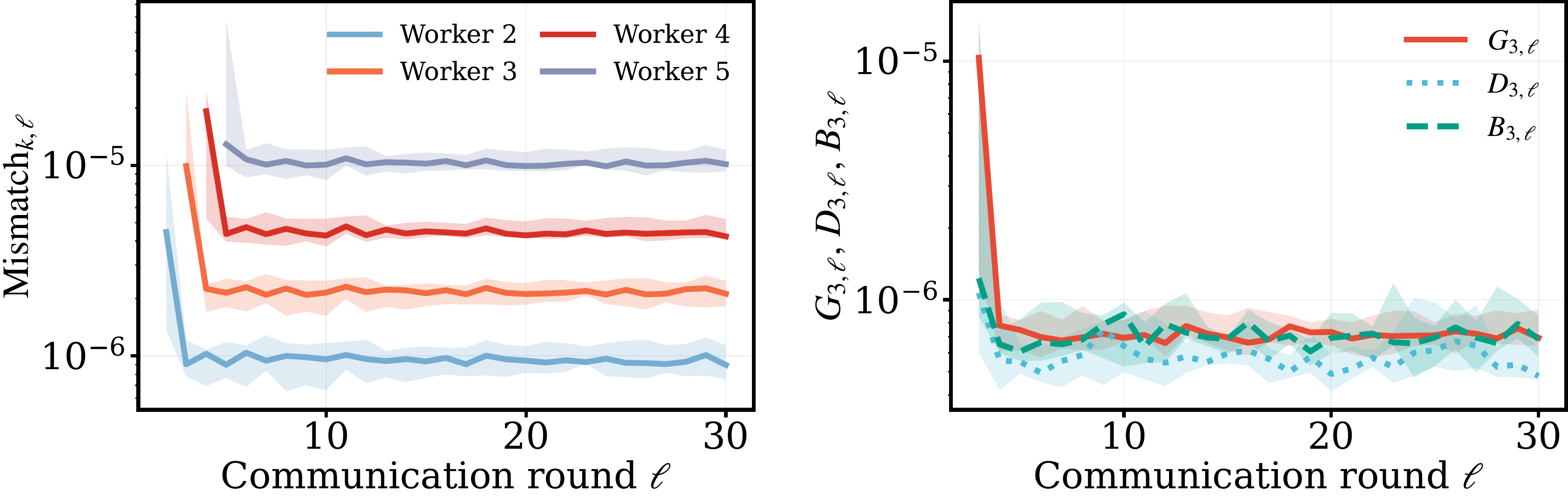}
\caption{\textbf{Self-correction.}
Left: normalized deflation mismatch per worker, decreasing over rounds.
Right: error decomposition for worker~3 ($G$, $D$, $B$ all decay).}\label{fig:appendix-self-correction}
\end{figure}

\begin{figure}[t]
\centering
\includegraphics[width=\textwidth]{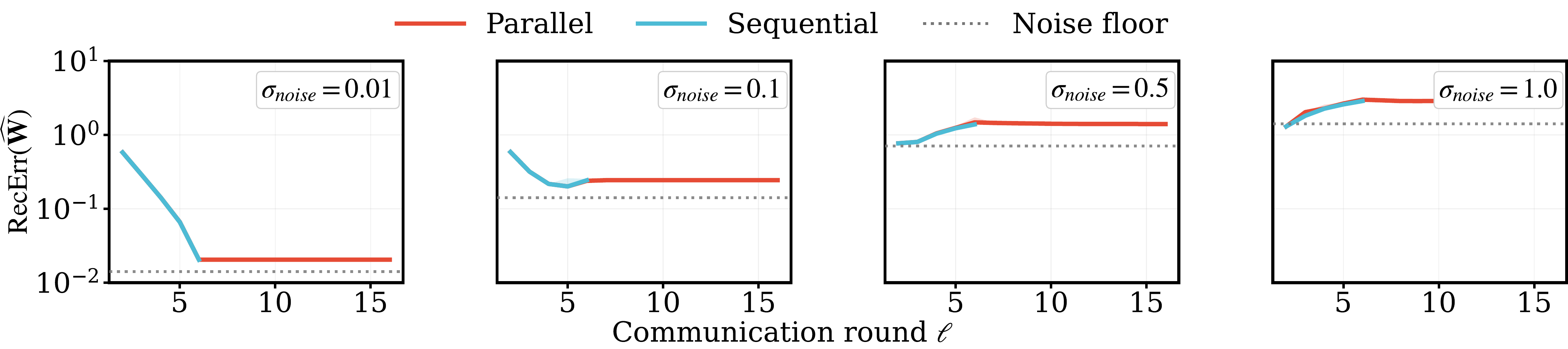}
\caption{\textbf{Noise robustness.}
Reconstruction error vs.\ rounds at four noise levels $\varepsilon$.
Parallel and sequential converge to the same noise floor.}\label{fig:appendix-noise}
\end{figure}

\subsection{Theoretical-Bound Validation}\label{app:exp_bounds}

This experiment tests the quantitative form of \cref{thm:parallel_convergence}: each worker's total error $G_{k,\ell}$ (defined in~\eqref{eq:def_G}) decays geometrically once the warm-up round $s_k$ from \cref{thm:parallel_convergence} is reached, and the empirical per-worker contraction rate matches the recurrence~\eqref{eq:mk_recurrence} for $m_k$.
We fix $m = 50$, $d = 80$, $n = 200$, $r^\star = 5$ on the exponential spectrum and run \cref{alg:parallel_rank1} with $T = 10$ ALS iterations per round and $L = 40$ communication rounds.
For every seed and every worker $k$ we record the empirical trajectory $G_{k,\ell}$ and fit $G_{k,\ell} \approx \widehat{C}_k\, \widehat{m}_k^{\,\ell}$ by least-squares regression on $\log G_{k,\ell}$, restricted to rounds at which $G_{k,\ell}$ has fallen below half of its value at worker activation.
Here $m_k$ denotes the per-worker contraction factor of the recurrence~\eqref{eq:mk_recurrence}, and $\calF_k$ denotes the per-component contraction factor of the rank-1 subroutine from \cref{asump:contraction}.
The empirical warm-up $\widehat{s}_k$ is taken as the first round after which $G_{k,\ell}$ decays monotonically over the next five rounds, and $\widehat{m}_k$ is the fitted slope; we use $\widehat{s}_k$ and $\widehat{m}_k$ as the empirical estimates of the theoretical $s_k$ and $m_k$.

\Cref{fig:appendix-bounds} compares these estimates to the theoretical predictions.
Panel~(A) overlays the fitted exponentials on the empirical $G_{k,\ell}$ curves and shows that every worker enters geometric decay after a brief warm-up.
Panel~(B) plots $\widehat{s}_k$ together with the staggered-activation prediction $s_k = k$ and the Lambert-$W$ upper bound from~\eqref{eq:sk_condition}: the observed warm-up tracks $s_k = k$, while the Lambert-$W$ bound is conservative on this well-separated spectrum, as expected of a worst-case bound (for visual clarity, the displayed curve uses the looser global-$m_k$ form of~\eqref{eq:sk_condition}, while the proof states a tighter per-predecessor maximum over $k' \leq k$; both are valid upper bounds on $s_k$).
Panel~(C) places the fitted $\widehat{m}_k$ next to the recurrence $m_k = \max\{\calF_k, 1/k + (k-1) m_{k-1}/k\}$ from~\eqref{eq:mk_recurrence}, and the two agree within seed-to-seed noise.
Together the three panels confirm that the convergence rate predicted by \cref{thm:parallel_convergence} is attained empirically and that the fitted contraction rate $\widehat{m}_k$ matches the recurrence~\eqref{eq:mk_recurrence} for $m_k$ within seed-to-seed variability.

\begin{figure}[t]
\centering
\includegraphics[width=\textwidth]{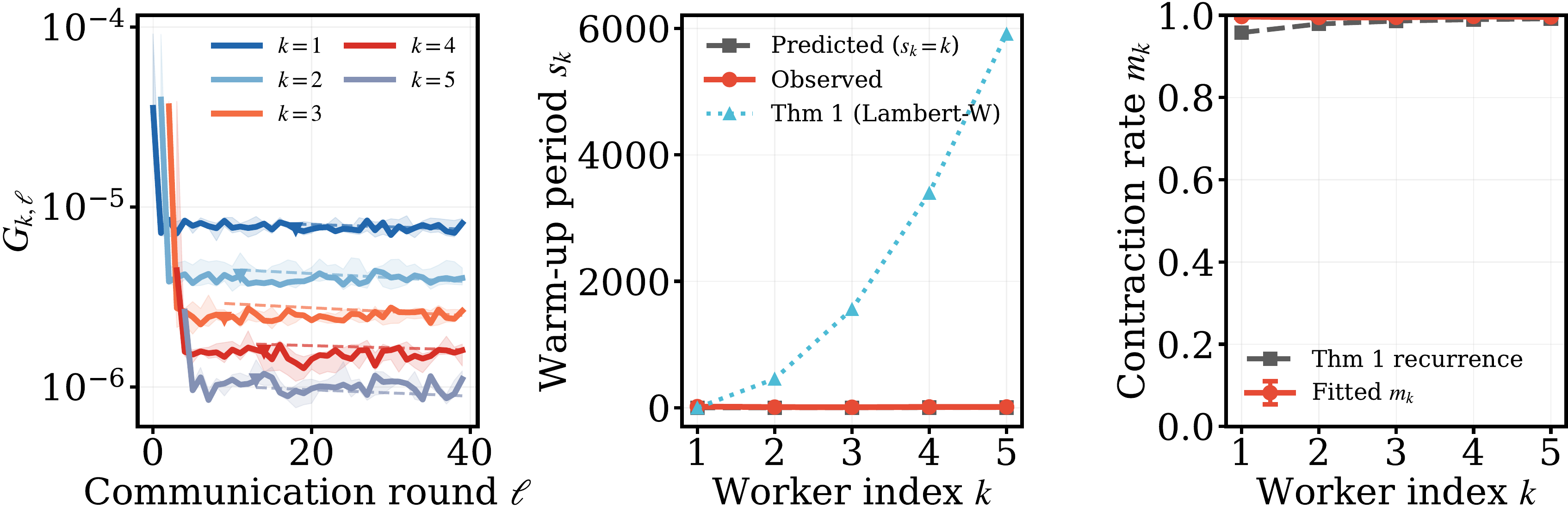}
\caption{\textbf{Theoretical-bound validation for \cref{thm:parallel_convergence}.}
(A) Empirical $G_{k,\ell}$ per worker (solid) with fitted exponential decay (dashed) and observed warm-up marker.
(B) Predicted ($s_k = k$) vs.\ observed warm-up period, with the Lambert-$W$ bound from~\eqref{eq:sk_condition} shown for reference.
(C) Fitted contraction rate $\widehat{m}_k$ vs.\ the recurrence~\eqref{eq:mk_recurrence}.}\label{fig:appendix-bounds}
\end{figure}

\subsection{Spectral-Gap Sensitivity}\label{app:exp_spectral_gap}

The third experiment probes \cref{asump:gap} directly by sweeping the spectral gap and tracking how convergence degrades as it closes.
We use a linear profile $\sigma_k^\star = 1 - (k-1) g$ for $k = 1, \ldots, r^\star$, clamped at a floor of $0.01$ so the problem stays well-posed when $g$ is small.
Here $g$ is the spectral-gap step parameter for this experiment, distinct from the growth interval $\Delta$ used in the main text (\cref{tab:app_hparams}).
The remaining settings match \cref{app:exp_bounds}: $m = 50$, $d = 80$, $n = 200$, $r^\star = 5$, $T = 10$, $L = 30$.
On this profile the per-worker gap from \cref{asump:gap}, $\calT_k^\star = \min\bigl(\min_{j>k}|\sigma_k^\star - \sigma_j^\star|, \sigma_k^\star\bigr)$, equals $\min(g, \sigma_k^\star)$ for $k < r^\star$, so over the swept range the ratio $g/\sigma_1^\star$ directly controls $\min_k \calT_k^\star / \sigma_1^\star$.
We sweep $g/\sigma_1^\star \in \{0.01, 0.05, 0.1, 0.25, 0.5\}$.

\Cref{fig:appendix-gap} reports the outcome.
Panel~(a) shows that all five gap settings converge and that the per-round rate degrades smoothly as the gap shrinks: at $g/\sigma_1^\star = 0.5$ the reconstruction error reaches its floor within roughly five communication rounds $\ell$, while at $g/\sigma_1^\star = 0.01$ it takes about twenty-five rounds.
Panel~(b) plots the across-seed median of the final-round error against the gap ratio on a log scale; even at $g/\sigma_1^\star = 0.01$ the error drops below $10^{-4}$, so the algorithm degrades gracefully rather than failing as the gap closes.
A positive spectral gap is therefore sufficient for convergence on this profile, and shrinking the gap inflates the warm-up period $s_k$ (\cref{thm:parallel_convergence}) smoothly rather than triggering a sharp phase transition, consistent with the role $\calT_k^\star$ plays in \cref{asump:gap}.

\begin{figure}[t]
\centering
\includegraphics[width=\textwidth]{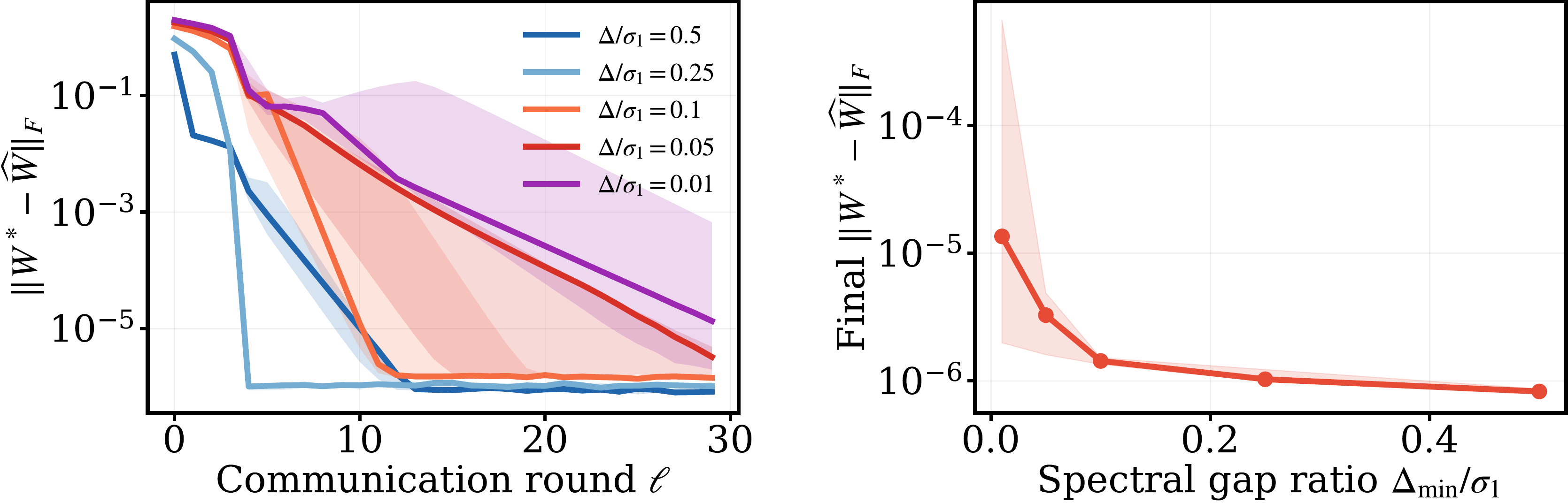}
\caption{\textbf{Spectral-gap sensitivity (validates \cref{asump:gap}).}
(a) Convergence curves for five gap ratios $g/\sigma_1^\star$; smaller gaps require more rounds but still converge.
(b) Final reconstruction error vs.\ gap ratio decays to a floor below $10^{-4}$ even at $g/\sigma_1^\star = 0.01$.}\label{fig:appendix-gap}
\end{figure}

\section{Additional Experimental Details}\label{app:experiments}

\subsection{GLUE Training Details}
\paragraph{Model and tokenizer.}
All GLUE experiments use DeBERTaV3-base~\citep{HeDeBERTaV3_2023} (\texttt{microsoft/deberta-v3-base}), a 12-layer, 768-hidden, 12-head transformer with ${\sim}184$M parameters.
We freeze all pretrained weights and train only the rank-1 adapter parameters.
Inputs are tokenized with the DeBERTaV3 SentencePiece tokenizer at maximum sequence length~128.

\paragraph{Adapter configuration.}
We attach parallel rank-1 adapters to all 72 encoder linear layers (six per layer $\times$ 12 layers): $W_q$, $W_k$, $W_v$, $W_o$ (attention) and $W_{f1}$, $W_{f2}$ (FFN).
Each adapter has $r_{\max} = 4$ components and both $\bfa_k$ and $\bfb_k$ vectors are initialized from $\mathcal{N}(0, 0.02)$.
We use the standard LoRA scaling $\alpha$~\citep{HuLoRA2022}, with the per-task values reported in \cref{tab:app_hparams}.

\iffalse % commented per Barbara: hyperparameters are enough; do not duplicate dataset stats
\paragraph{Reproducibility.}
All GLUE experiments use a single random seed (seed=43) and a single development-set evaluation per task. We do not use checkpoint averaging, exponential moving averages, or ensembling.

\paragraph{GLUE dataset statistics.}
\Cref{tab:glue_stats} summarizes the eight GLUE tasks~\citep{Wang2019GLUE}.
All are single-sentence or sentence-pair classification, except STS-B which is regression.

\begin{table}[h]
\centering
\caption{GLUE dataset statistics.}\label{tab:glue_stats}
\small
\begin{tabular}{llrrll}
\toprule
Corpus & Task & Train & Dev & Metrics & Domain \\
\midrule
CoLA & acceptability & 8.5k & 1k & Matthews corr. & misc. \\
RTE & NLI & 2.5k & 276 & Acc. & news, Wikipedia \\
MRPC & paraphrase & 3.7k & 408 & Acc. & news \\
STS-B & sentence similarity & 7k & 1.5k & Pearson/Spearman & misc. \\
SST-2 & sentiment & 67k & 872 & Acc. & movie reviews \\
QNLI & QA/NLI & 105k & 5.4k & Acc. & Wikipedia \\
QQP & paraphrase & 364k & 40k & Acc. & social QA \\
MNLI & NLI & 393k & 9.8k & Acc. & misc. \\
\bottomrule
\end{tabular}
\end{table}
\fi

\paragraph{Per-task hyperparameters.}
\Cref{tab:app_hparams} lists the hyperparameters used for each task on the GLUE benchmark~\citep{Wang2019GLUE}; LoRA's scaling convention follows~\citet{HuLoRA2022}.

\begin{table}[h]
\centering
\caption{Per-task hyperparameters for \adapad{} on GLUE (seed=43).}\label{tab:app_hparams}
\footnotesize
\setlength{\tabcolsep}{3pt}
\resizebox{\textwidth}{!}{%
\begin{tabular}{lcccccccc}
\toprule
 & CoLA & RTE & MRPC & STS-B & SST-2 & QNLI & QQP & MNLI \\
\midrule
Epochs & 10 & 50 & 30 & 25 & 24 & 5 & 7 & 9 \\
Batch size & 32 & 32 & 32 & 32 & 64 & 64 & 128 & 256 \\
Learning rate & $8\!\times\!10^{-4}$ & $1.2\!\times\!10^{-3}$ & $1\!\times\!10^{-3}$ & $1.97\!\times\!10^{-3}$ & $1\!\times\!10^{-3}$ & $1\!\times\!10^{-3}$ & $7\!\times\!10^{-4}$ & $1.5\!\times\!10^{-3}$ \\
$\alpha$ & 32 & 24 & 24 & 32 & 8 & 32 & 48 & 24 \\
Weight decay & 0.01 & 0.01 & 0.01 & 0.01 & 0.01 & 0.01 & 0.01 & 0.01 \\
Warmup steps & 100 & 100 & 100 & 100 & 1000 & 500 & 500 & 500 \\
Grad clip & 1.0 & 1.0 & 1.0 & 1.0 & 1.0 & 1.0 & 1.0 & 1.0 \\
Top-$h$ & 3 & 3 & 6 & 1 & 5 & 3 & 5 & 2 \\
$\Delta$ (growth interval) & 100 & 200 & 100 & 100 & 100 & 50 & 300 & 100 \\
Orth reg coef $\lambda_{\text{orth}}$ & 0.3 & 0.3 & 0 & 0.5 & 0.2 & 0.2 & 0.15 & 0.1 \\
Adapter dropout & 0.0 & 0.0 & 0.3 & 0.0 & 0.1 & 0.0 & 0.1 & 0.2 \\
Cls dropout & 0.0 & 0.0 & 0.0 & 0.2 & 0.0 & 0.0 & 0.15 & 0.15 \\
Adapter weight decay & 0.0 & 0.01 & 0.01 & 0.0 & 0.0 & 0.0 & 0.0 & 0.0 \\
LR schedule & linear & linear & linear & linear & linear & linear & linear & cosine \\
Rank budget $B$ & 144 & 144 & 144 & 144 & 144 & 144 & 144 & 144 \\
\bottomrule
\end{tabular}%
}
\end{table}

% \paragraph{Parameter budget.}
% Each rank-1 component adds $d_{\text{in}} + d_{\text{out}} + 1$ parameters per module: vectors $\mathbf{a} \in \mathbb{R}^{d_{\text{in}}}$, $\mathbf{b} \in \mathbb{R}^{d_{\text{out}}}$, and scalar $\lambda$.
% With $r_{\max} = 4$ and 72 modules, \adapad{} allocates $0.66$M adapter parameters during training.
% At convergence the average rank is ${\approx}2$, so only ${\sim}0.33$M parameters remain active in the forward pass, matching IncreLoRA's parameter count at target rank~2.
% The ``\#Params'' column in \cref{tab:glue_main} reports the nominal trainable-parameter budget per method: LoRA at $r{=}2$ uniformly, \adapad{} and IncreLoRA at target average rank~2.

\paragraph{Orthogonal regularization.}
We optionally include the orthogonal regularization term from IncreLoRA~\citep{ZhangIncreLoRA2023}, $\lambda_{\text{orth}} \bigl(\|\bfA^\top \bfA - \bfI\|_F + \|\bfB \bfB^\top - \bfI\|_F\bigr)/(2 n_{\text{layers}})$, which pushes each module's component vectors toward an orthonormal frame.
We adopt it on GLUE with the per-task values in \cref{tab:app_hparams} and disable it on SQuAD.

\paragraph{Baseline reproduction.}
We reproduce IncreLoRA as the primary comparison baseline because it is the closest method to \adapad{}: both grow rank incrementally from rank-1 components via importance-based allocation with EMA scoring.
We use the authors' official loralib (SVDLinear + RankAllocator) and their published hyperparameters, ported to our software environment (Python~3.11, PyTorch~2.6, Transformers~5.2, CUDA~12.9).
For the remaining baselines (AdaLoRA, SoRA, dEBORA) we cite published results from their respective papers~\citep{ZhangAdaLoRA2023,DingSoRA2023,dEBORA2025}.

\subsection{SQuAD/SQuAD~v2 Training Details}\label{app:squad_experiments}

\paragraph{Model and tokenizer.}
We fine-tune Qwen/Qwen3-0.6B~\citep{Qwen3Report2025} with the standard \texttt{AutoModelForQuestionAnswering} span-prediction head.
The backbone has 28 transformer blocks with hidden size $1024$, FFN dim $3072$, and grouped-query attention with $16$ query heads and $8$ KV heads.
Inputs are tokenized at sequence length $384$ with doc-stride $128$, following the HuggingFace canonical preprocessing for both SQuAD~v1.1~\citep{Rajpurkar2016SQuAD} and SQuAD~v2~\citep{Rajpurkar2018SQuADv2}.

\paragraph{Adapter configuration.}
Both \adapad{} and the fixed-rank LoRA~\citep{HuLoRA2022} baseline adapt all seven projection matrices in every transformer block: $W_q$, $W_k$, $W_v$, $W_o$, $W_{\text{gate}}$, $W_{\text{up}}$, $W_{\text{down}}$, for a total of $7\times28{=}196$ adapted modules.
\adapad{}'s rank-1 components initialise both $\bfa_k$ and $\bfb_k$ vectors from $\mathcal{N}(0, 0.02)$.
We set $r_{\max}{=}8$ and allocate the remaining budget by EMA importance scoring with the top-$10$ modules per growth round.

\paragraph{Per-task hyperparameters.}
\Cref{tab:app_squad_hparams} lists the single \adapad{} recipe used across all four (task, $\bar r$) cells. Only the rank budget changes between cells, driven by the average-rank target; all other hyperparameters are shared.

\begin{table}[h]
\centering
\caption{\adapad{} hyperparameters on SQuAD/SQuAD~v2 (seed=43).
A single recipe is used for all four cells; only the rank budget $B$
changes between $\bar r{=}4$ and $\bar r{=}6$.}\label{tab:app_squad_hparams}
\small
\begin{tabular}{lcc}
\toprule
 & SQuAD~v1.1 & SQuAD~v2 \\
\midrule
Epochs                                & 3 & 3 \\
Batch size                            & 16 & 16 \\
Max sequence length                   & 384 & 384 \\
Doc-stride                            & 128 & 128 \\
Total optimiser steps                 & 16{,}650 & 24{,}777 \\
Learning rate                         & $1\!\times\!10^{-3}$ & $1\!\times\!10^{-3}$ \\
$\alpha$                              & 16 & 16 \\
Weight decay                          & 0.0 & 0.0 \\
Warmup steps                          & 1{,}000 & 1{,}000 \\
Grad clip                             & 1.0 & 1.0 \\
Top-$h$                               & 10 & 10 \\
$\Delta$ (batches per growth round)   & 400 & 400 \\
Orth reg coef                         & 0.0 & 0.0 \\
Adapter dropout                       & 0.0 & 0.0 \\
LR schedule                           & linear & linear \\
$r_{\max}$                            & 8 & 8 \\
Rank budget $B$ ($\bar r{=}4$ / $\bar r{=}6$)  & 784 / 1{,}176 & 784 / 1{,}176 \\
Optimiser                             & AdamW (bf16) & AdamW (bf16) \\
\bottomrule
\end{tabular}
\end{table}

\iffalse % commented per Barbara: parameter counts are already implicit in \cref{tab:causal_lm_main}
\paragraph{LoRA baseline.}
LoRA~\citep{HuLoRA2022} uses the standard PEFT recipe at uniform rank $r\in\{4,6\}$ on the same seven projection matrices: batch size $16$, sequence length $384$, learning rate $1\!\times\!10^{-3}$, $\alpha{=}r{\times}2$, AdamW (bf16), and a linear schedule with $1{,}000$ warmup steps over 3 epochs.
Trainable adapter parameter counts are $2{,}525{,}186$ at $r{=}4$ and $3{,}786{,}754$ at $r{=}6$.
\fi

% \paragraph{Reporting convention for \cref{tab:causal_lm_main}.}
% For each cell we record the per-module rank at every evaluation checkpoint and compute the deployed adapter parameter count as the sum of $d_{\text{in}}+d_{\text{out}}+1$ over committed components ($2{,}049$ for $W_k, W_v$; $3{,}073$ for $W_q, W_o$; $4{,}097$ for the FFN projections in Qwen3-0.6B).
% \adapad{}'s reported row is the earliest checkpoint at which (i) the deployed adapter has fewer parameters than the corresponding LoRA baseline and (ii) F1 matches or exceeds LoRA's. EM at the same step is reported as-is.

\subsection{Hardware and Software}

Unless otherwise stated, all experiments run on $4 \times$ NVIDIA H200 80\,GB GPUs connected by NVLink on a single node.
Software: Python~3.11, PyTorch~2.6, Transformers~5.2, CUDA~12.9.

\end{document}